\documentclass{article}

\usepackage[noend,boxruled,linesnumbered]{algorithm2e}
\usepackage{setspace}
\SetKwInput{KwInput}{Input}                
\SetKwInput{KwOutput}{Output}              
\SetKwInput{KwTarget}{Target}
\usepackage[preprint]{neurips_2021}

\usepackage{times}
\usepackage{soul}
\usepackage[utf8]{inputenc}
\usepackage[T1]{fontenc}    
\usepackage{color}
\usepackage{epsfig}
\usepackage{algorithmic}
\usepackage{amsthm}
\usepackage{amsmath}
\usepackage{amssymb}
\usepackage{float}
\usepackage{graphicx}
\usepackage{wrapfig}
\usepackage{makecell}
\usepackage{caption}
\usepackage{subcaption}
\usepackage{bm}
\usepackage{multirow}
\usepackage{microtype}
\usepackage{mathtools}
\usepackage{ulem}
\usepackage{xcolor}
\usepackage{pseudocode}

\newcommand{\kevin}[1]{{\color{cyan} #1\xspace}}
\newcommand{\jk}[1]{\textcolor{blue}{[Jiankai: #1]}}

\newcommand{\synzero}{{label-majority}}
\newcommand{\synzerotr}{{label-majority-tr}}
\newcommand{\synzerote}{{label-majority-te}}

\newcommand{\synone}{{label-minority}}
\newcommand{\synposratio}{{label-random-pos}}
\newcommand{\synpredprob}{{label-random-pred}}
\newcommand{\synknn}{{label-neighbors}}

\newcommand{\baselinefull}{{baseline-full}}
\newcommand{\baselinepartial}{{baseline-partial}}

\newcommand{\feaSampling}{{fea-sampling}}
\newcommand{\feaSamplingTr}{{fea-sampling-tr}}
\newcommand{\feaSamplingTe}{{fea-sampling-te}}

\newcommand{\feaRandom}{{fea-random}}
\newcommand{\feaGaussian}{{fea-Gaussian}}

\newcommand{\ourappvanilla}{{FLORIST-vanilla}}
\newcommand{\ourapp}{{FLORIST}}
\newcommand{\actp}{\mathcal{P}_a}
\newcommand{\pasp}{\mathcal{P}_p}
\newcommand{\actid}{\mathcal{I}_a}
\newcommand{\pasid}{\mathcal{I}_p}

\usepackage{commands}       

\title{Vertical Federated Learning without Revealing Intersection Membership}

\author{Jiankai Sun\thanks{Equal contribution. Correspondence to: \{jiankai.sun, yangxin.yx, chong.wang\}@bytedance.com } \\ jiankai.sun@bytedance.com \And Xin Yang\footnotemark[1] \\ yangxin.yx@bytedance.com \AND Yuanshun Yao \\ kevin.yao@bytedance.com \\ \And
Aonan Zhang \\ aonan.zhang@bytedance.com \AND Weihao Gao \\ weihao.gao@bytedance.com \And Junyuan Xie \\ junyuan.xie@bytedance.com \And  Chong Wang \\ chong.wang@bytedance.com}  

\begin{document}

\maketitle

\begin{abstract}


\textit{Vertical Federated Learning} (vFL) allows multiple parties that own different attributes (\textit{e.g.} features and labels) of the same data entity (\textit{e.g.} a person) to jointly train a model. To prepare the training data, vFL needs to identify the common data entities shared by all parties. It is usually achieved by \textit{Private Set Intersection} (PSI) which identifies the intersection of training samples from all parties by using personal identifiable information (\textit{e.g.} email) as sample IDs to align data instances. As a result, PSI would make sample IDs of the intersection visible to all parties, and therefore each party can know that the data entities shown in the intersection also appear in the other parties, \textit{i.e.} intersection membership. However, in many real-world privacy-sensitive organizations, \textit{e.g.} banks and hospitals, revealing membership of their data entities is prohibited. In this paper, we propose a vFL framework based on \textit{Private Set Union} (PSU) that allows each party to keep sensitive membership information to itself. Instead of identifying the intersection of all training samples, our PSU protocol generates the union of samples as training instances. In addition, we propose strategies to generate synthetic features and labels to handle samples that belong to the union but not the intersection. Through extensive experiments on two real-world datasets, we show our framework can protect the privacy of the intersection membership while maintaining the model utility.

\end{abstract}

\section{Introduction}
\label{sec:intro}


With the increasing tension between data privacy and data-hungry machine learning, \textit{Federated Learning} (FL)~\citep{mcmahan2017communication,Filip2020,yuan2020,Avishek2020} is proposed as a privacy-enhancing technique that allows multiple parties to collaboratively train a model without completely sharing data. Depending on how data are split across parties, FL can be mainly classified into two categories~\citep{ylct19}: \textit{Horizontal Federated Learning}~\citep{Jonas2020,Jenny2020,Sai2020,Li2020} and \textit{Vertical Federated Learning}~\citep{vgsr18,gr18,akk+20,csm+20}. In Horizontal FL, data is split by entity (\textit{e.g.} a person), and data entities owned by each party are complete and disjoint from other parties. In Vertical FL (vFL), a data entity is split into different attributes (\textit{e.g.} features and labels of the same person), and each party might own the same data entity but different attributes. One typical example of vFL is a collaboration between general and specialized hospitals. They might hold the data for the same patient, but the general hospital owns generic information (\textit{i.e.} features) of the patient while the specialized hospital owns the specific testing results (\textit{i.e.} labels) of the same patient. Therefore they can use vFL to jointly train a model that predicts a specific disease examined by the specialized hospital from the features provided by the general hospital.

One critical stage in vFL is identifying the same entities shared by all parties which we define as \textit{intersection}. In standard vFL, this is achieved by \textit{Private Set Intersection} (PSI) protocols ~\citep{PSI2016vlad,PSI2016Benny}. Before the training starts, all parties need to run PSI protocols to identify the intersection to align sample IDs. The sample ID needs to be certain universally identifiable information that can be used to identify entities across organizations. The common option is some personal identifiable information (\textit{e.g.} phone and email). After all parties obtain the sample IDs of the intersection, they can jointly train a model on the intersection.

However, the PSI protocol would make the sample IDs of the intersection visible to all parties. As a result, every party would know that the entities in the intersection also exist in the other parties' data, and we define this information as \textit{intersection membership}. In many real-world privacy-sensitive organizations, membership information is highly sensitive and cannot be shared with the other parties. For example, hospitals cannot reveal which patient is a member and banks cannot disclose which client owns an account. In practice, this leakage significantly limits the applicability of vFL since vFL participants are privacy-sensitive but the membership leakage conflicts with the privacy-preserving intention of using vFL.

Only a few of existing work attempts to address the intersection membership leakage problem in vFL. 
Existing work in protecting the privacy of vFL mostly focuses on preventing data leakage~\citep{cjsy20,wcxco20,labelleakage} rather than membership leakage.
The closest work is~\citep{lzw20}, which adapts PSI to achieve the asymmetrical ID alignment in an asymmetrical vFL. However, their protocol still exposes the intersection membership to one party and cannot protect all parties. To the best of our knowledge, our work is the first attempt to protect intersection membership information for all parties.

In this paper, we design a novel vFL framework to address the problem of intersection membership leakage. Our framework can train models without revealing the membership information while maintaining the model utility. To achieve our goal, we need to overcome two main technical challenges: 1) How to securely align training samples without the knowledge of their intersection? 2) How to design an effective learning mechanism based on the secured alignments to protect privacy while retaining the model utility?

To answer the questions, we propose a novel \textit{Private Set Union} (PSU) protocol that does not reveal sample IDs of the intersection, and therefore no membership information about the intersection would be leaked. Instead, it identifies the union of the training instances and uses it as the training data in our framework. In addition, we design several strategies to generate synthetic features and labels for samples that belong to the union but not the intersection. We summarize our contributions as follows:

\begin{itemize}
    \item We propose the first vFL framework that  protects the membership information of all parties.
    \item We design a novel Private Set Union protocol that can securely align data entities without revealing membership information of the intersection. 
    \item We design several synthetic data generation strategies for samples which belong to the union but not the intersection. Through extensive experiments, we show our strategies are both secure and utility-preserving.
\end{itemize}

\section{Methodology}

\label{sec:methodology}
\subsection{Background: Vertical Federated Learning}


We formally describe the gradient-based two-party vFL settings. Let $\mathcal{D} = (\mathcal{I}, \mathcal{X}, \mathcal{Y})$ denote a complete dataset with $\mathcal{I}$, $\mathcal{X}$, and $\mathcal{Y}$ representing the sample ID space, the feature space, and the label space, respectively. A two-party vFL is conducted over two datasets $\mathcal{D}_p = (\mathcal{I}_p, \mathcal{X}_p, \mathcal{Y}_p)$, $\mathcal{D}_a = (\mathcal{I}_a, \mathcal{X}_a, \mathcal{Y}_a)$, satisfying $\mathcal{X}_p \ne \mathcal{X}_a, \mathcal{Y}_p = \emptyset, \mathcal{Y}_p \ne \mathcal{Y}_a, \mathcal{I}_p \cap \mathcal{I}_a \ne \emptyset$. We refer the party $\mathcal{P}_p$ without labels ($\mathcal{Y}_p = \emptyset$) as the \textit{passive party} and the party $\mathcal{P}_a$ with labels as the \textit{active party}. 

Here we focus on two parties learning a model for a binary classification problem over the domain $\calX \times \cbrck{0, 1}$. The passive and active parties want to learn a composition model $h \circ f$ jointly, where the raw features $X$ and $f: \calX \to \Real^d$ are stored on the passive party side while the labels $y$ and $h: \Real^d \to \Real$ is on the active party side. Let $\ell = h(f(X))$ be the logit\footnote{In our case, we add no additional features in the active party to compute the logit and set $\mathcal{X}_a = \emptyset$.} of the positive class where the positive class's predicted probability is given by the sigmoid function. The loss $L$ of the model is given by the cross entropy. The passive party sends the computation result, $f(\mathcal{X}_p)$ of the intermediate layer (called the \textit{cut layer}) rather than the raw data $\mathcal{X}_p$ to the active party. The active party computes the gradient of $L$ with respect to the input of the function $h$. We denote this gradient by $g$  (equality by chain rule). After receiving $g$ sent from the active party, the passive party computes the gradient of $L$ w.r.t. $f$'s parameters. It is straightforward to apply stochastic gradient descent (SGD) to this setting.


As discussed in Section~\ref{sec:intro}, before we  run the training algorithm, we need to prepare data by first aligning the training data using 
PSI \citep{PSI2016Benny, PSI2016vlad}.
However this will inevitably leak intersection membership.

\subsection{Overview of \ourapp{}}
\label{sec:framework}

We propose {\textbf{F}}ederated {\textbf{L}}earning with{\textbf{O}}ut {\textbf{R}}evealing {\textbf{I}}nter\textbf{S}ec{\textbf{T}}ions, or  {\textbf{FLORIST}}, to remove the need of revealing the intersection membership in vFL. 
Specifically, we perform mini-batch SGD training on the \textit{union} of the training samples. As shown in Figure~\ref{fig:framework_overview}, our framework consists of two modules: Private Set Union (PSU) for ID alignment and synthetic data generation. PSU is designed to align data samples securely without exposing sample IDs of the intersection.
PSU outputs a set of universal IDs with which both parties schedule the mini-batch training. First, in PSU, no information on the intersection set, except for the size of the intersection, is revealed. 
Second, when a party (either passive or active) is given a training sample $({id}_i, {x}_i, {y}_i)$ in a mini-batch generated from the union, it is possible that $id_i$ is not in its dataset. Therefore the party needs to generate synthetic data to ensure a valid training procedure. For example, as shown in Figure \ref{fig:framework_overview}, the passive party provides synthetic features for $[p_4]$, and the active party provides synthetic labels for $[p_2]$ and $[p_3]$, where $[.]$ represents the encrypted ID space. 

\paragraph{Threat Model.} We assume malicious parties are honest-but-curious, \textit{i.e.} $\actp$ and $\pasp$ faithfully run the vFL protocol, but they may infer important information including raw sample data and intersection membership from the exchanged information.
In particular, we consider two privacy leakage scenarios:
1) for $id\in \actid$, 
$\actp$ finds out if $id \in \actid\cap \pasid$ by checking $y_{id}$ and the embedding forwarded by $\pasp$,
and 2) for $id\in \pasid$,
$\pasp$ finds out if $id \in \actid\cap \pasid$ by checking $x_{id}$ and the gradient sent back by $\actp$. 


\begin{figure*}[ht!]
  \centering
    \includegraphics[width=0.76\linewidth]{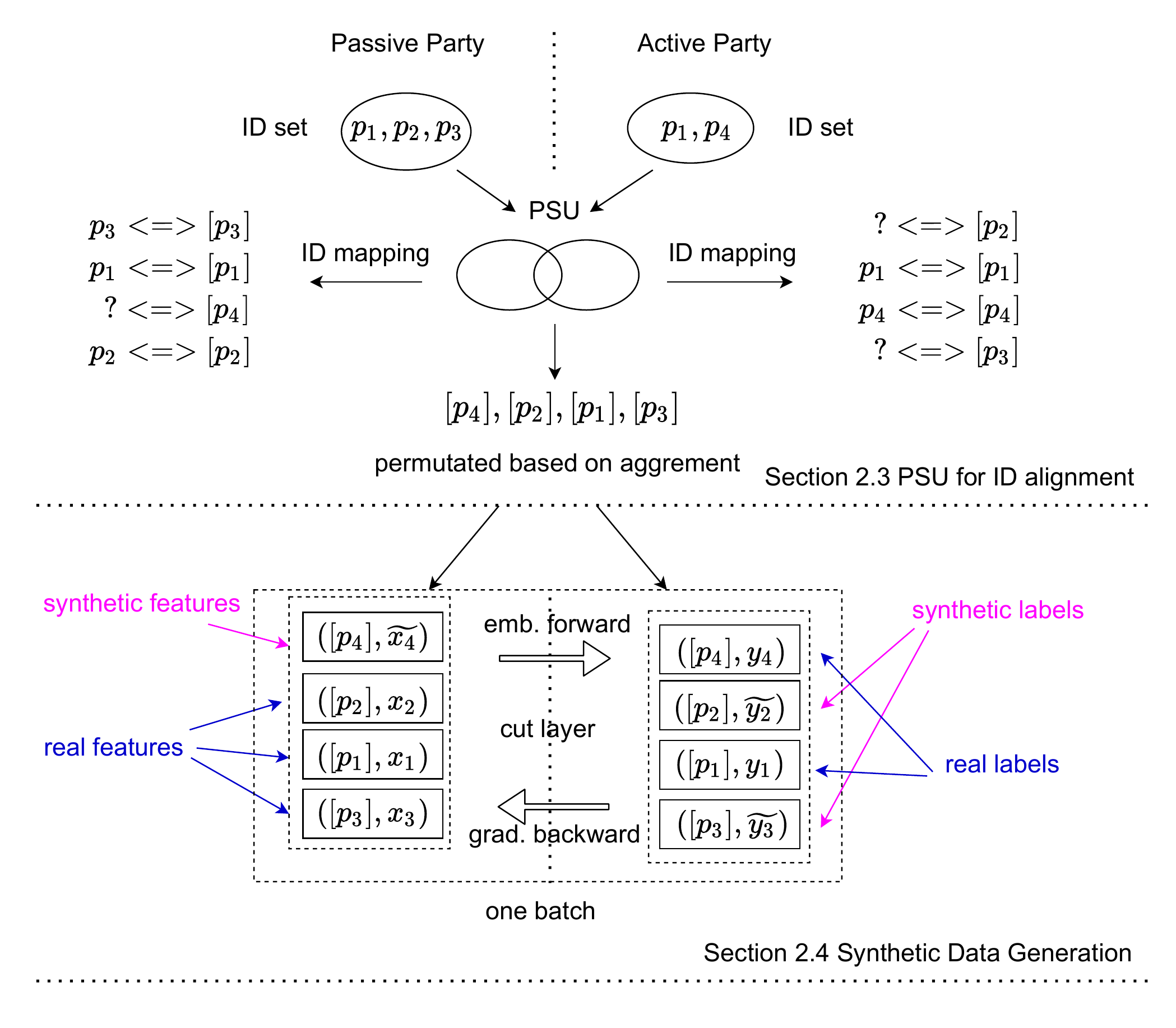}
 \caption{Overview of \ourapp{}. $[p_i]$ is ID $p_i$'s encrypted UID. $\widetilde{x_i}$ and $\widetilde{y_i}$ indicate that those sample IDs use synthetic features and labels respectively. \ourapp{} consists of two modules: PSU for ID alignment (up) and synthetic data generation (bottom). See main texts for details of each component.}
 \label{fig:framework_overview} 
 \vspace{-0.15in}
\end{figure*}

\SetAlgoSkip{}
\LinesNumberedHidden{
\begin{algorithm}[t]
\setstretch{1.3}
\caption{Private Set Union Protocol}
\label{alg:psu}
\hskip0.5em \textbf{Inputs:} $\actp$ holds $\actid$, $\pasp$ holds $\pasid$, and $\actid,\pasid\subset QR(\mathbb{Z}_p^*)$ where $p$ is a prime number,
$q=\frac{p-1}{2}$ is also a prime number,
and $QR(\mathbb{Z}_p^*)$ is the quadratic residues of $\mathbb{Z}_p^*$. All the operations are taken modulus $p$.\\
\hskip0.5em \textbf{Protocol:}\\
 \hskip2em \textbf{Initialization:} \\
 \hskip3em 1. $\actp$ generates random $s_1,s_2,s_3\in \mathbb{Z}_q$\\
{} \hskip3em 2. $\pasp$ generates random $t_1,t_2,t_3\in \mathbb{Z}_q$\\

{} \hskip2em \textbf{First Round Hashing:} \\
{} \hskip3em 3. $\actp$ computes $\actid^{s_1} \equiv \{x^{s_1}|x\in \actid\}$ , randomly shuffles $\actid^{s_1}$, and sends it to $\pasp$\\
{} \hskip3em 4. $\pasp$ computes $\actid^{s_1t_1} \equiv \{x^{t_1}|x\in \actid^{s_1}\}$, randomly shuffles $\actid^{s_1t_1}$, and sends it to $\actp$\\
{} \hskip3em 5. $\pasp$ computes $\pasid^{t_1} \equiv \{x^{t_1}|x\in \pasid\}$, randomly shuffles $\pasid^{t_1}$, and sends it to $\actp$\\
{} \hskip3em 6. $\actp$ computes $\pasid^{s_1t_1} \equiv \{x^{s_1}|x\in \pasid^{t_1}\}$, randomly shuffles $\pasid^{s_1t_1}$, and sends it to $\pasp$\\

{} \hskip2em \textbf{Second Round Hashing:} \\
{} \hskip3em 7. $\actp$ merges $\actid^{s_1t_1}$ and $\pasid^{s_1t_1}$ to obtain the list $(\actid\cup \pasid)^{s_1t_1}$\\
{} \hskip3em 8. $\actp$ computes $(\actid\cup \pasid)^{s_1s_2s_3t_1} \equiv \{x^{s_2s_3}|x\in (\actid\cup \pasid)^{s_1t_1}\}$, randomly shuffles $(\actid\cup \pasid)^{s_1s_2s_3t_1}$, and sends it to $\pasp$\\
{} \hskip3em 9. $\pasp$ computes $(\actid\cup \pasid)^{s_1s_2s_3t_1t_2t_3} \equiv \{x^{t_2t_3}|x\in (\actid\cup \pasid)^{s_1s_2s_3t_1}\}$, randomly shuffles $(\actid\cup \pasid)^{s_1s_2s_3t_1t_2t_3}$, and sends it to $\actp$\\
{} \hskip3em 10. Both $\actp$ and $\pasp$ sort $(\actid\cup \pasid)^{s_1s_2s_3t_1t_2t_3}$ in lexicographic order to obtain UID set $U \equiv \{u_1,\cdots,u_m\}$ where $m = |\actid\cup \pasid|$\\

{} \hskip2em \textbf{Computing Private Hashing:} \\
\hskip3em 11. For $x\in \actid$ \\
 \hskip6em $\actp$ computes $x^{s_2}$, and sends it to $\pasp$\\
 \hskip6em  On receiving $y:=x^{s_2}$, $\pasp$ computes $x^{s_2t_1t_2t_3}=y^{t_1t_2t_3}$, and sends it to $\actp$\\
 \hskip6em  On receiving $z:=x^{s_2t_1t_2t_3}$, $\actp$ computes $x^{s_1s_2s_3t_1t_2t_3}=z^{s_1s_3}$\\
 \hskip6em  $\actp$ saves the bijective mapping $x\leftrightarrow x^{s_1s_2s_3t_1t_2t_3}$
 
\hskip3em 12. For $x\in \pasid$ \\
 \hskip6em $\pasp$ computes $x^{t_2}$, and sends it to $\actp$\\
 \hskip6em On receiving $y:=x^{t_2}$, $\actp$ computes $x^{s_1s_2s_3t_2}=y^{s_1s_2s_3}$, and sends it to $\pasp$\\
 \hskip6em On receiving $z:=x^{s_1s_2s_3t_2}$, $\pasp$ computes $x^{s_1s_2s_3t_1t_2t_3}=z^{t_1t_3}$\\
 \hskip6em $\pasp$ saves the bijective mapping $x\leftrightarrow x^{s_1s_2s_3t_1t_2t_3}$

\hskip0.5em \textbf{Outputs:} Both $\actp$ and $\pasp$ have UID set $U \equiv \{u_1,\cdots,u_m\}$ where $m = |\actid\cup \pasid|$. $\actp$ has the mapping $\actid\rightarrow U$ and $\pasp$ has the mapping $\pasid\rightarrow U$
\end{algorithm}
}

\subsection{Private Set Union for ID-alignment}
\label{sec:private_union}

We first introduce our PSU protocol for secured entity alignment. On a high level, by running the PSU protocol,
parties $\actp$ and $\pasp$ jointly compute a bijection $u: (\mathcal{I}_a \cup \mathcal{I}_p)\leftrightarrow U$,
where $U$ is a set of encrypted universal identifiers (UID) shared by both parties.
Furthermore,
the active party $\actp$ knows $u(id)$ for all $id\in\actid$ and
the passive party $\pasp$ knows $u(id)$ for all $id \in\mathcal{I}_p$.
Therefore the membership information that can be inferred from sample IDs aligned by PSI would not be leaked by PSU.
Finally, under the decisional Diffie–Hellman assumption,
our PSU protocol leaks no information other than $|\mathcal{I}_a |,|\mathcal{I}_p| ,|\mathcal{I}_a \cap \mathcal{I}_p|$ 
to an honest-but-curious malicious party. 

The set $U$ is used to schedule mini-batch training in the following way: for each batch, both parties agree on a sequence $u_1,u_2,\cdots\in U$.
For each $u_i$,
both parties check whether $u^{-1}(u_i)$ is contained in their own ID set,
and if so, they feed the corresponding real data into the model;
otherwise, they apply the data generation methods (detailed described in Section~\ref{sec:synthetic_data_generation}) to generate synthetic data to feed into the model.

Our PSU protocol is presented in Algorithm \ref{alg:psu}.
It consists of two parallel independent parts:
in the first part, both parties jointly compute the union of UID, but they do not know the mapping from the ID set to $U$.
We use two rounds of Diffie–Hellman key exchange scheme so that 
$\actp,\pasp$ cannot identify if some sample ID is in the intersection.
In the second part,
$\actp$ and $\pasp$ compute the mapping from $\actid,\pasid$ to $U$ respectively.


We present the security guarantee of our PSU protocol.
The security assumption is based on the hardness of the Decisional Diffie–Hellman (DDH) problem.
Let $p$ be a prime satisfying that $q=\frac{p-1}{2}$ is also a prime. ($p$ is a so-called \textit{safe prime}).
It is commonly believed that DDH is hard for $QR(\mathbb{Z}_p^*)$,
where $QR(\mathbb{Z}_p^*):=\{x|\exists y, y^2\equiv x \pmod{p}\}$ is the quadratic residues of the cyclic group $\mathbb{Z}_p^*$.
Here we assume that both $\actid$ and $\pasid$ have already been pre-processed so that they are encoded in $QR(\mathbb{Z}_p^*)$.

Our security argument is based on \textit{simulation},
which is a standard proof technique for proving privacy against honest-but-curious adversary~\cite{l17}. 
Here we present an informal statement and more detailed analysis can be found in Appendix~\ref{appendix:psu}.
\begin{theorem}[Security Guarantee of PSU]
Let $n:=\max\{|\actid|,|\pasid|,\lceil \log p \rceil\}$ be the security parameter.
Assume DDH is hard for $QR(\mathbb{Z}_p^*)$, then for malicious $\mathcal{P}\in \{\actp,\pasp\}$ that runs in time polynomial in $n$,
$\mathcal{P}$ learns nothing other than $|\actid|,|\pasid|,|\actid\cup\pasid|$ by running Algorithm \ref{alg:psu}.
\end{theorem}

\subsection{Synthetic Data Generation}
\label{sec:synthetic_data_generation}
Given a training sample $({id}_i, {x}_i, {y}_i)$ in a mini-batch generated from the union of all samples, it is possible that $id_i$ is not in $\actid$ or $\pasid$. In such cases, $\actp$ or $\pasp$ need to generate synthetic data. We introduce strategies to generate synthetic labels for active party and synthetic features for passive party respectively. 
The proposed strategies are designed to prevent the malicious parties from distinguishing which data point is synthetic while maintaining the model performance.

\paragraph{Synthetic Labels Generation for Passive Party.} The gradient $g$ sent to the passive party is a matrix in $\mathbb{R}^{B \times d}$ with each row belonging to a specific sample ID in the batch with size $B$. Here the gradients as rows of the matrix are gradients of the loss with respect to intermediate computation results of different samples. 
Synthetic gradients should be provided for samples which are not in the sample ID space ($\mathcal{I}_a$) of the active party to protect privacy. There are two ways to generate synthetic gradients: 1) generate synthetic labels firstly and use the corresponding labels to compute training loss and then generate the corresponding gradients; 2) generate synthetic gradients directly without 
setting labels. We design several strategies such as k-nearest neighbors to generate synthetic gradients from actual gradients and empirically find these strategies insecure, i.e., the passive party can distinguish synthetic gradients from real ones.  Therefore we focus on generating synthetic labels. 

The synthetic label generation strategy is based on the following observation: the datasets in many real-world applications such as online advertising and healthcare are imbalanced. The negative samples significantly outnumber the positive ones. For example, Criteo
\footnote{{https://www.kaggle.com/c/criteo-display-ad-challenge/data}} 
and Avazu
\footnote{{https://www.kaggle.com/c/avazu-ctr-prediction/data}} are two real-world large-scale binary classification datasets, and
both datasets are highly imbalanced: only $25\%$ of the Criteo and $17\%$ of the Avazu samples are positive. Hence, given a binary classification problem, it's reasonable to assign the missing labels as negative ($0$ as the label). 

\paragraph{Synthetic Features Generation for Active Party.}

For sample IDs not in $\mathcal{I}_p$, the passive party generates synthetic raw features for them. A naive strategy is to generate random values as synthetic features. However, random values and real features have different distributions;
the active party can utilize tools like clustering and outlier detection from robust statistics to decide which features are synthetic and then $\mathcal{I}_p$ is leaked to the active party. Empirically we demonstrated the effectiveness of using SVD-based outlier detection method~\cite{tlm18} to distinguish random values from real features. 

We present our method that best protects the privacy. We show that, theoretically, the best strategy for the $\pasp$ is to generate synthetic features according to the marginal distribution of real data. Let $D(x,y)$ be the (unknown) ground truth distribution of the feature-label pairs whose IDs are not owned by $\pasp$, and $q(\cdot)$ be the distribution of labels of those data. As $\pasp$ does not have access to labels, its generation strategy can be characterized by a distribution $p(\cdot)$. That is, whenever the passive party needs to generate a synthetic feature, it samples from $p$.
Because features and labels are independent for synthetic data, the joint distribution is of the form $D'(x,y)=p(x)\cdot q(y)$.  In order to have the best privacy, we want $D,D'$ to be as close as possible. 
The following theorem suggests that the marginal distribution is optimal in minimizing KL-divergence. The proof is included in Appendix~\ref{appendix:kl_proof}.
\begin{theorem}[folklore]
\label{theorem_min_kl}
Let $p$ be a distribution over $\mathcal{X}$ and $q$ be a distribution over $\mathcal{Y}$.
Let $D'$ be a joint distribution over $\mathcal{X}\times \mathcal{Y}$ given by $D'(x,y)=p(x)q(y)$.
Let $p^*$ be a distribution over $\mathcal{X}$ given by $p^*(x)=\int_y D(x,y)dy$.
Then $\text{KL}(D||D')$ is minimized when $p=p^*$. 
\end{theorem}

Nevertheless, $\pasp$ has no access to $p^*$ because $\pasp$ simply does not have these features.
Here we make the assumption that the data distribution is the same across data corresponding to ID sets $\actid \cap \pasid$, $\actid-\pasid$ and $\pasid-\actid$.
Then $\pasp$ could think of its feature set $\mathcal{X}_p$ follows the distribution of $p^*$.

As it is difficult to accurately estimate marginal distribution on input features (a high dimension distribution), we use the following sampling method to generate synthetic features: the passive party uses a random real sample to fill in the non-existing features. The raw features (including real and synthetic) will be fed to the same network and generate the intermediate embeddings at the cut layer.

\paragraph{Logits Calibration.}
As our data generation scheme
changes the data distribution,
it can cause large adaptive calibration (ACE) error~\citep{ACE2019}. ACE measures how well a model’s predicted probabilities of outcomes reflect true probabilities of those outcomes. We adapt a method named \textit{logits shift}~\citep{Chapelle2015ads} to handle the calibration error.
On a high level, since both parties know how synthetic data is generated,
they can estimate how the data distribution deviates from the underlying one, which allows them to modify the model to account for ACE. Empirically we demonstrated that this method can reduce the ACE significantly. Details of how to do the calibration via logits shift can be seen in the appendix.

\section{Empirical Study}
\label{sec:experiments}

In this section, we experimentally evaluate the proposed framework. Since PSU is an alignment protocol and it is proved to be secure, the main goal of the empirical study is to demonstrate the effectiveness of the data generation strategies.
We run three groups of simulations: 1) \textbf{Simulations with Synthetic Labels Only}. During training, 
partial labels of the active party are synthetic. We measure the performance of our proposed synthetic label generation with varying synthetic label ratio $\beta$. 2) \textbf{Simulations with Synthetic Features Only}. During training, 
partial features provided by the passive party are synthetic. We check the performance with varying synthetic feature ratio $\alpha$. 3) \textbf{Simulations with Both Synthetic Labels and Features}. During training, partial data of both parties are synthetic. We check our performance with varying $\alpha$ and $\beta$.

\paragraph{Dataset and Model.} We use two real-world datasets: Criteo 
 and Avazu. 
Criteo is a large-scale binary classification dataset with approximately $45$ million user click records in online advertising. Avazu contains approximately $40$ million entries ($11$ days of clicks/not clicks of Avazu data). We defer the similar results on Avazu to Appendix~\ref{sec:exp_avazu} and only report the results of Criteo in this section. We split each dataset randomly into two parts: with $90\%$ for training and the rest for tests.
 We train a modified Wide\&Deep model \citep{cheng2016wide} where the passive party consists of embedding layers for input features and two layers of $128$-unit ReLU activated multilayer perceptron (half deep part) and the active party consists of the last two layers of the deep part. In every iteration, the passive party sends a mini-batch of $8, 192$ examples' $128$-dimensional vectors to the active party and the active party sends the gradients of the loss w.r.t. these vectors back to the passive side. 


\paragraph{Evaluation Metrics.} We measure the model utility by the AUC and ACE. 
A model with an ideal utility should have a high AUC and a low ACE. We
also measure the security of the synthetic data generation stage since an honest-but-curious party could try to infer which data samples from the other party are synthetic or real. Therefore it might introduce new membership leakage. To this end, theoretically we can estimate the difference between the distribution $D$ of actual data and the distribution $D'$ of synthetic data. Yet these distributions are over high-dimensional data and it is intractable to compute the distance. Instead, we use spectral attack \citep{tlm18}, a 2-clustering outlier detection algorithm that can be used to predict if an embedding or gradient is real or synthetic. The AUC of its prediction, which we call it as \textit{attack AUC}, measures how well the malicious parties could distinguish real data and synthetic data generated by our strategies. A closer to $0.5$ attack AUC is considered to be more secure. More details on the spectral attack can be found in Appendix~\ref{sec:spectral_attack}.

\subsection{Simulations with Synthetic Labels Only}
\label{sec:exp_synthetic_label_only}

We conduct experiments to simulate the scenario that only the active party has to provide synthetic labels for unowned instances during training. The passive party owns all the sample IDs and hence it provides real features for all training instances. Based on the observation that both Criteo and Avazu datasets are highly imbalanced (only $25\%$ of the samples in Criteo and $17\%$ in Avazu are positive), our \texttt{\synzero{}} generates synthetic label $\widetilde{y_i} = 0$,  $\forall \ {id}_i \notin \mathcal{I}_a$. We also test several other synthetic label generation strategies for the active party: $\forall \ {id}_i \notin \mathcal{I}_a$,  1) \texttt{\synone} sets $\widetilde{y_i} = 1$; 2) \texttt{\synposratio} randomly samples a label for $\widetilde{y_i}$, based on the positive instance ratio of $\mathcal{D}_a$; 3) \texttt{\synpredprob} randomly samples a label for $\widetilde{y_i}$, based on the predicted positive probability for ${id}_i$; 4) \texttt{\synknn} determines $\widetilde{y_i}$, based on a k-nearest-neighbors strategy (with $k=3$ and cosine similarity to measure the similarity between cut layer embeddings).


We test different synthetic label ratio $\beta$ and report performance from different generation strategies. We compare the above generation strategies with two baselines. 1) \texttt{\baselinepartial}: Given a $\beta$, it only uses the rest $(1-\beta)$ real labels to train the model; 2) \texttt{\baselinefull}: it uses the whole real dataset to train the model, which can be viewed as an upper-bound for all competitors. 

For each $\beta$, every experiment uses the same fixed random seed to select unowned instances of the passive party. All strategies have the same real data as used in \baselinepartial{} to train the model. The corresponding results of AUC and ACE are shown in Figure ~\ref{fig:criteo_utility_privacy} (a) and (d) respectively. We can observe that using less data decreases the model performance. Unsurprisingly, the gap between \baselinepartial{} and \baselinefull{} gradually enlarges with increasing $\beta$. For example, when $\beta = 95\%$,  \baselinepartial{} drops about $3.8\%$ in comparing with \baselinefull. 


As shown in Figure \ref{fig:criteo_utility_privacy} (a), \synzero{} can achieve comparable AUC with \baselinepartial{}.
However, \synzero{} changes the data distribution and causes high ACE. As shown in Figure \ref{fig:criteo_utility_privacy} (d), \synzero{} has a larger ACE than baselines, and therefore it is necessary to calibrate the predicted logits for \synzero{}. Depending on when to have the calibration, we provide two variants of \synzero{}: \texttt{\synzerote{}} and \texttt{\synzerotr{}}. For \synzerote{}, we leave the training logits unchanged and perform the calibration in the testing. For \synzerotr{}, we only calibrates the predicted logits in the training. More details are included in Appendix \ref{sec:logits_shift_synthetic_label}.

As shown in Figure \ref{fig:criteo_utility_privacy} (d), we find both options reduce the ACE significantly and achieve similar ACE with baselines. Regarding the AUC, \synzerotr{} only drops about $0.67\%$ and $0.33\%$ in comparing with \baselinepartial{} when $\beta$ is $95\%$ and $50\%$ respectively. Meanwhile, as shown in Table \ref{tab:syn_label_criteo_leak_auc}, the attack AUC of \synzero{} and its variants are all around $0.57$. Though \synpredprob{} has the best privacy among all competitors (detailed analysis can be seen in Appendix \ref{sec:syn_label_pred_prob}), its utility is worse than \synzero{}.  
Therefore our proposed \synzero{} is highly effective in both protecting privacy and maintaining model utility.

\begin{table}[t]\centering
\caption{Attack AUC of different synthetic label generation strategies on Criteo.}\label{tab:syn_label_criteo_leak_auc}\scriptsize
\begin{tabular}{c|cccccccc}\toprule
$\beta=0.5$ &.-majority &.-majority-tr &.-majority-te &.-minority &.-random-pos &.-random-pred &.-neighbors \\ \midrule
Attack AUC &0.5761 &0.5768 &0.5764 &0.6511 &0.5342 &0.5074 &0.5445 \\
\bottomrule
\end{tabular}
\vspace{-0.15in}
\end{table}

\begin{figure*}[ht!]
\vspace{-0.1in}
  \begin{minipage}[b]{0.33\linewidth}
  \centering
    \includegraphics[width=\linewidth]{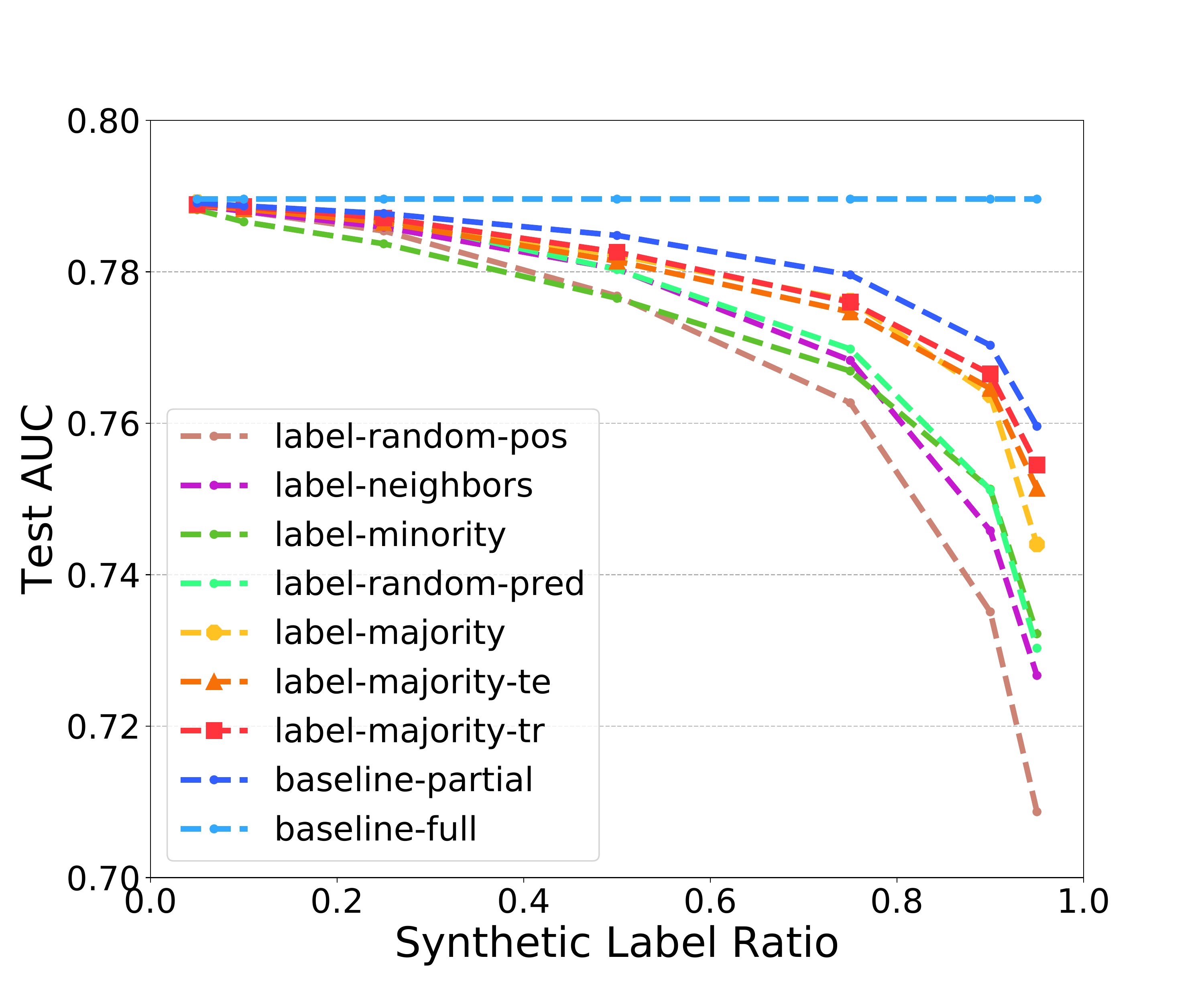}
      \caption*{ (a): Synthetic Label AUC}
  \end{minipage}
  \begin{minipage}[b]{0.33\linewidth}
  \centering
    \includegraphics[width=\linewidth]{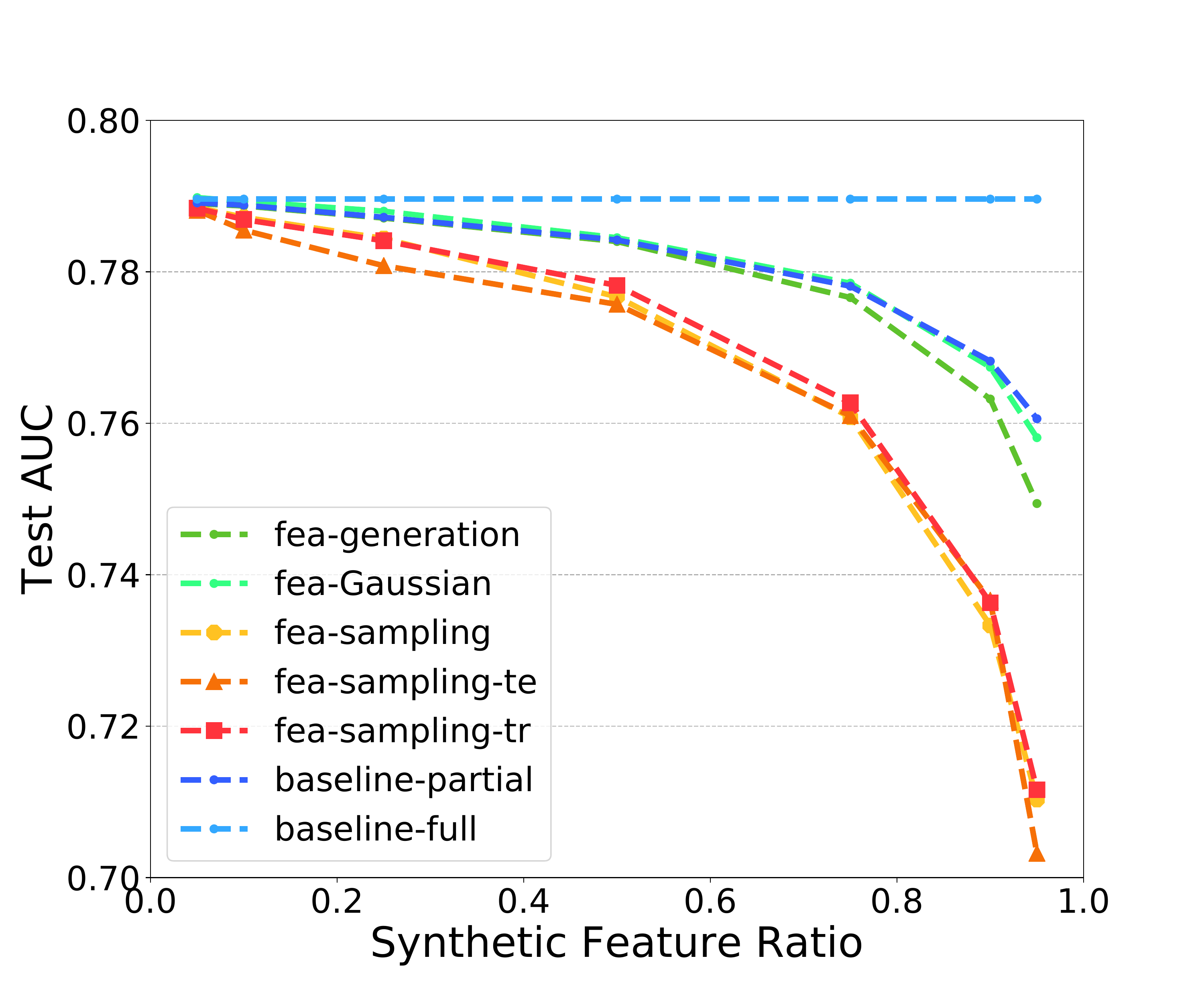}
    \caption*{(b): Synthetic Feature AUC}
  \end{minipage}
   \begin{minipage}[b]{0.33\linewidth}
  \centering
    \includegraphics[width=\linewidth]{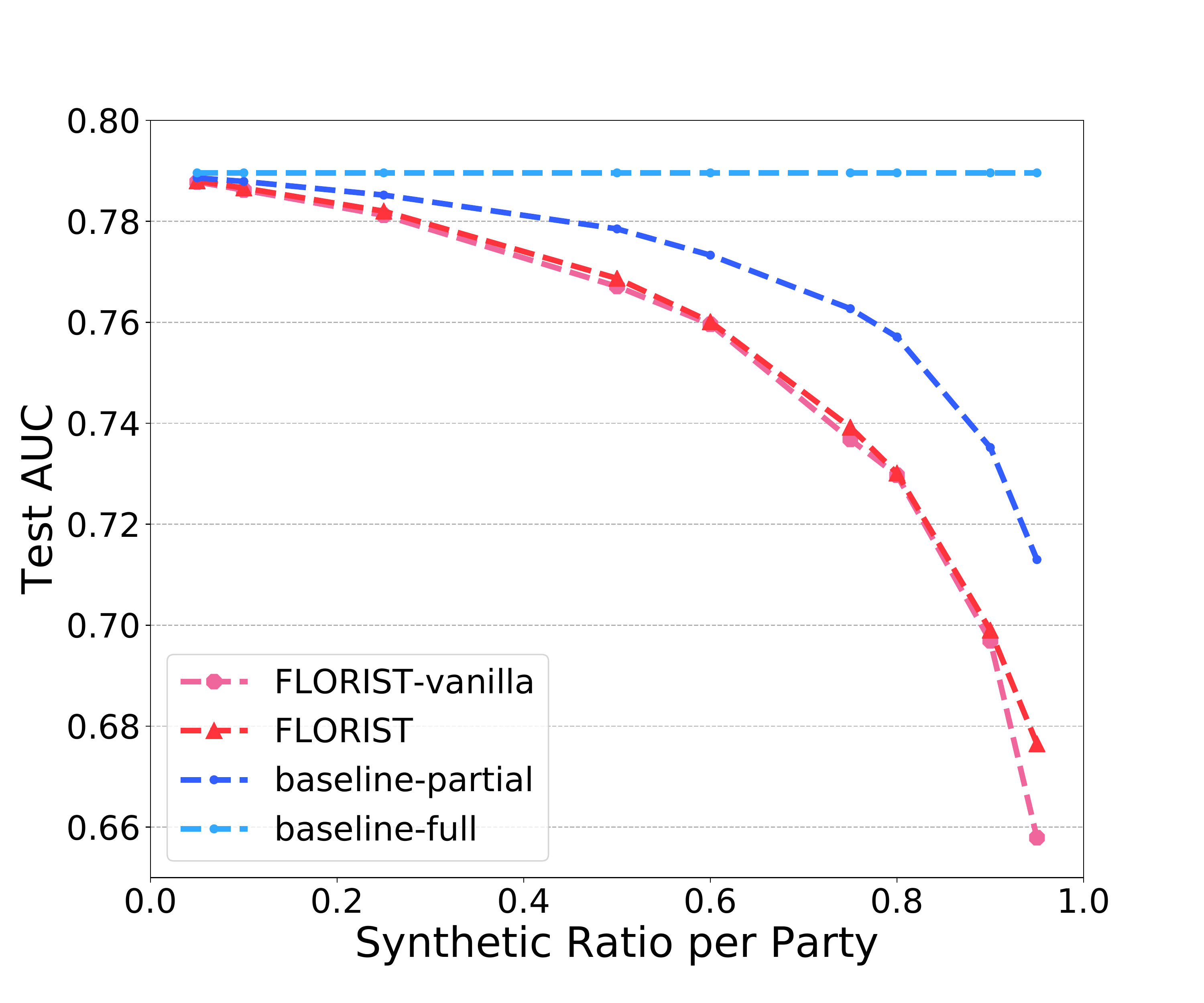}
      \caption*{(c): Synthetic Both AUC}
  \end{minipage}
  \medskip
    \begin{minipage}[b]{0.33\linewidth}
  \centering
    \includegraphics[width=\linewidth]{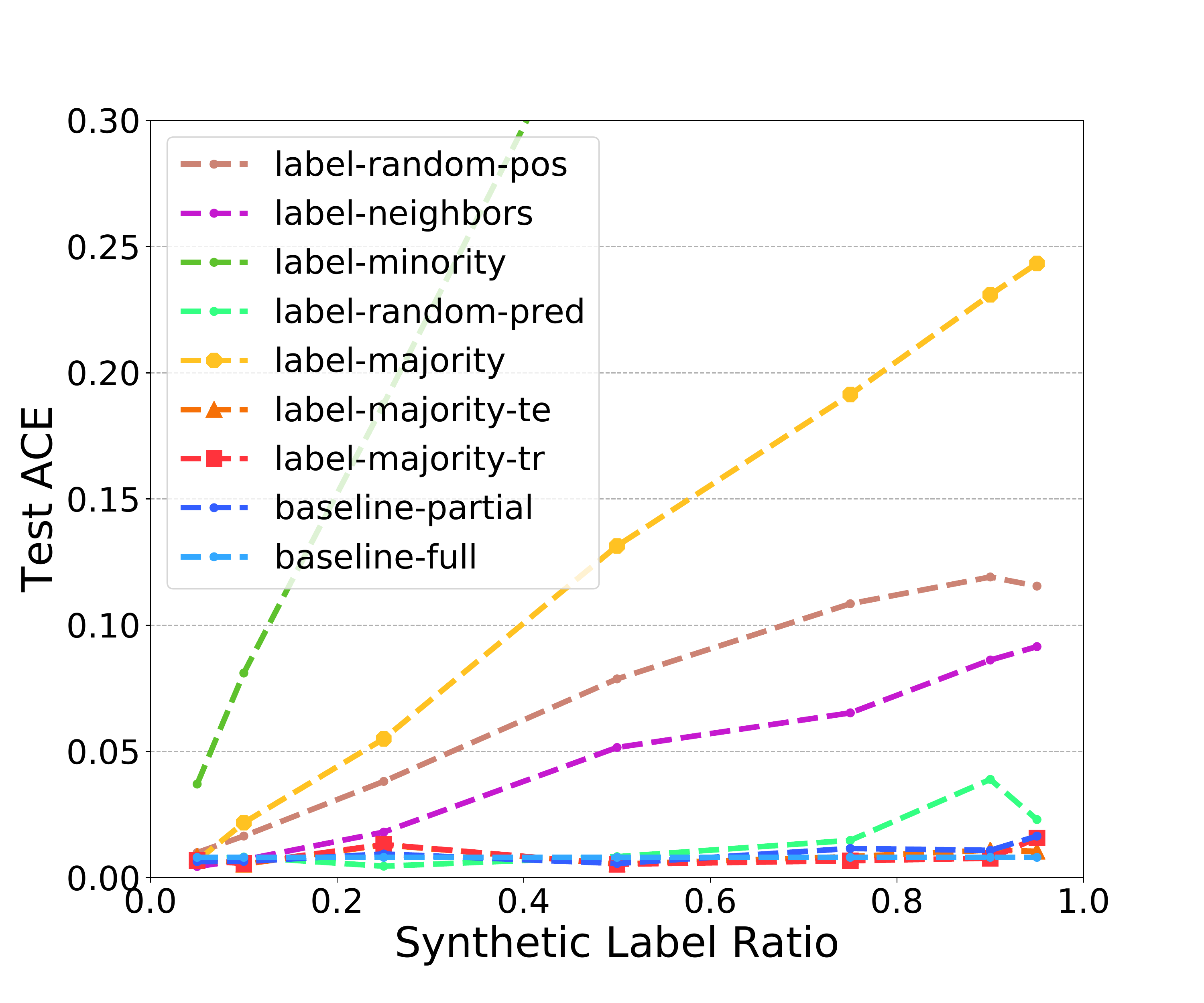}
      \caption*{(d): Synthetic Label ACE}
  \end{minipage}
  \begin{minipage}[b]{0.33\linewidth}
  \centering
    \includegraphics[width=\linewidth]{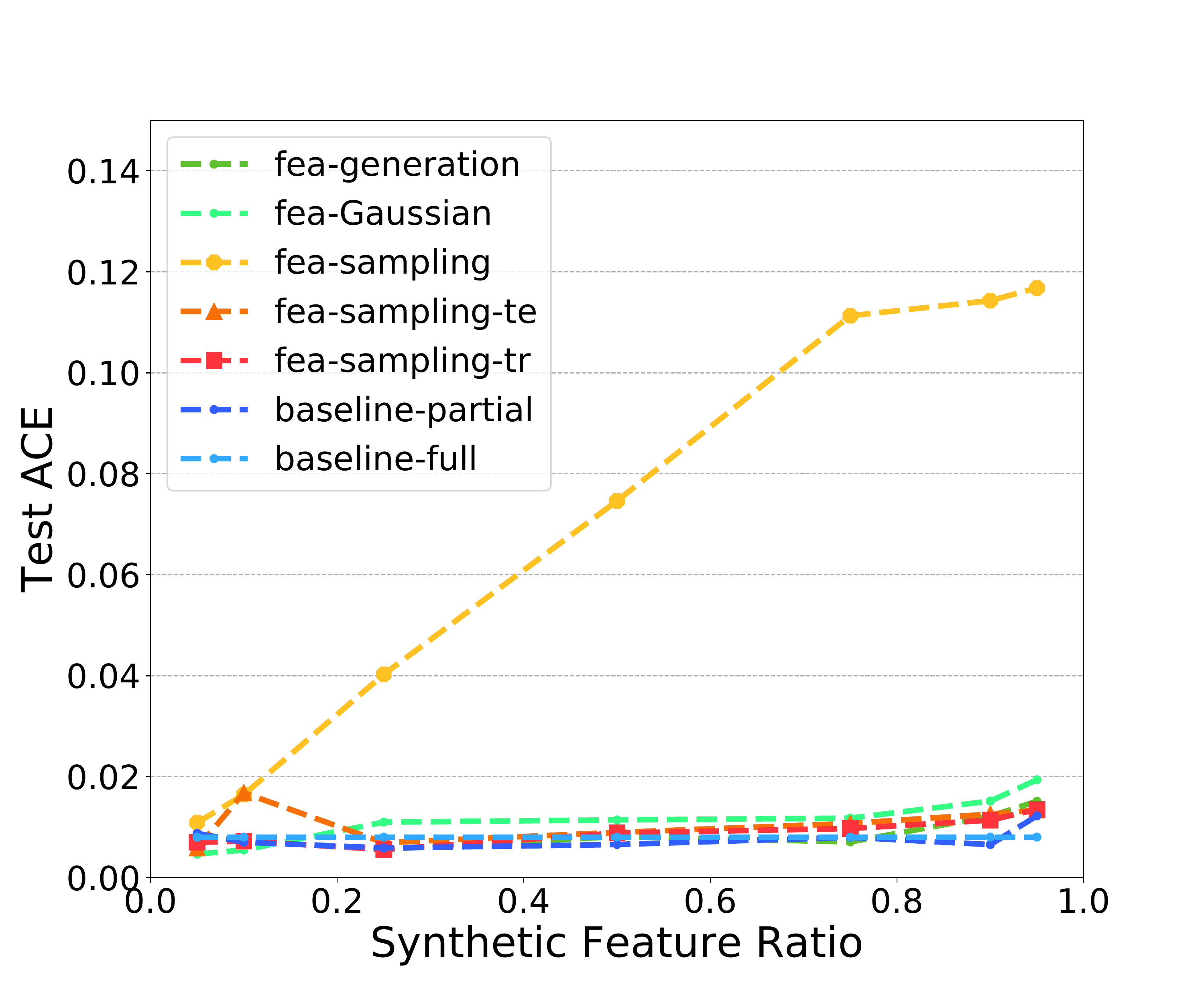}
    \caption*{(e): Synthetic Feature ACE}
  \end{minipage}
   \begin{minipage}[b]{0.33\linewidth}
  \centering
    \includegraphics[width=\linewidth]{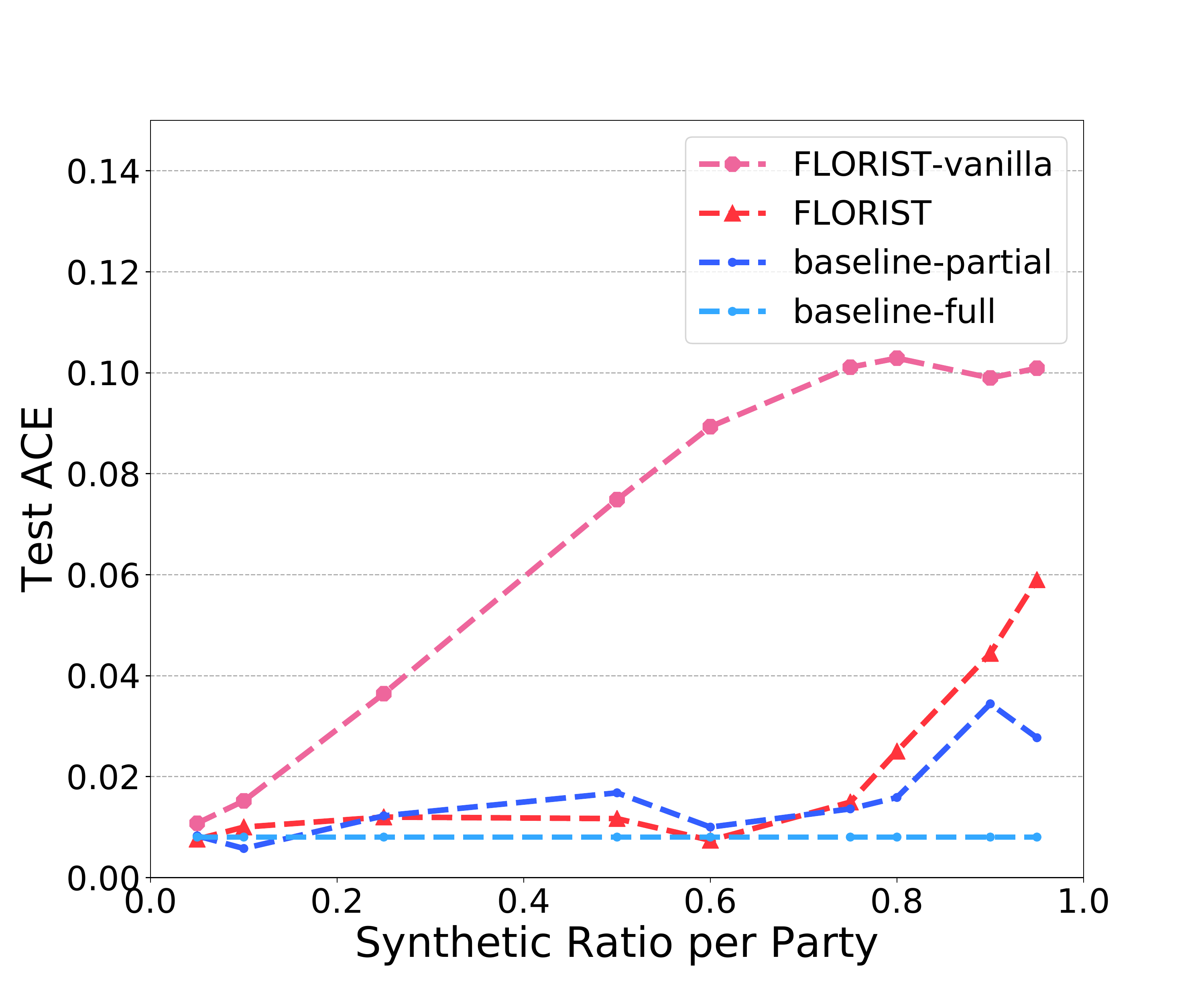}
      \caption*{(f): Synthetic Both ACE}
  \end{minipage}
  \medskip
    \begin{minipage}[b]{0.33\linewidth}
  \centering
    \includegraphics[width=\linewidth]{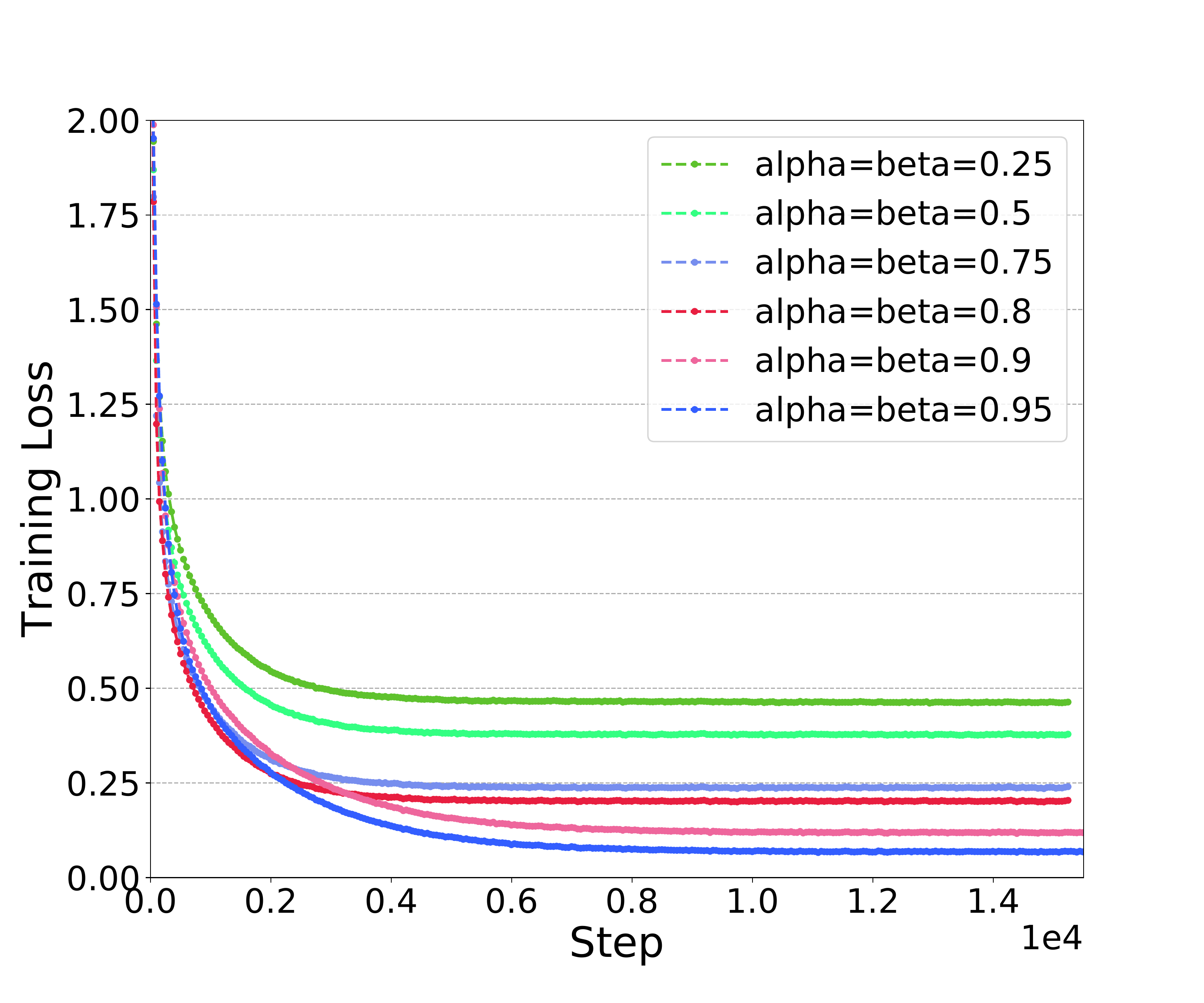}
      \caption*{(h): Training Loss of \ourapp{}}
  \end{minipage}
  \begin{minipage}[b]{0.33\linewidth}
  \centering
    \includegraphics[width=\linewidth]{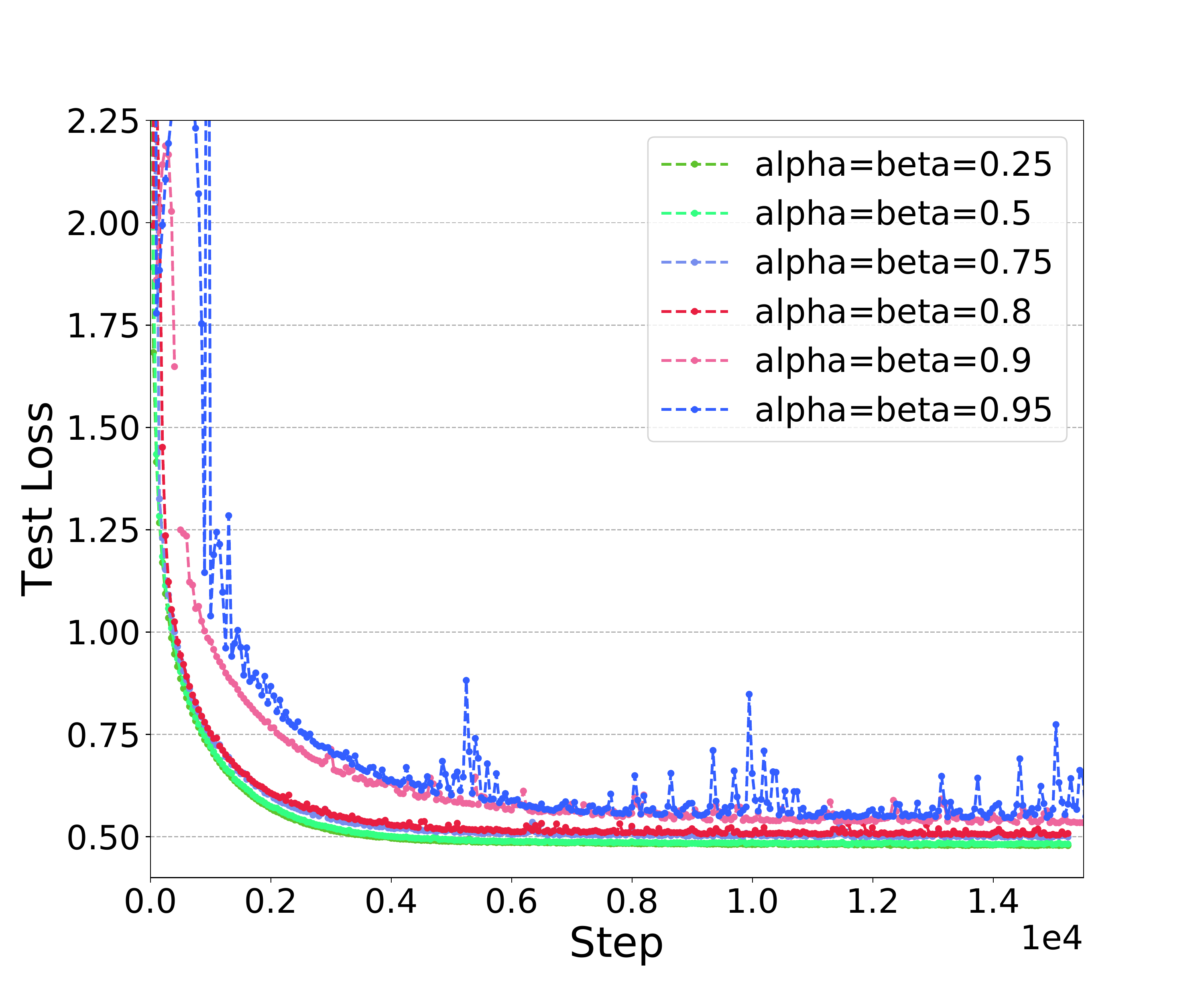}
    \caption*{(i): Test Loss of \ourapp{}}
  \end{minipage}
   \begin{minipage}[b]{0.33\linewidth}
  \centering
    \includegraphics[width=\linewidth]{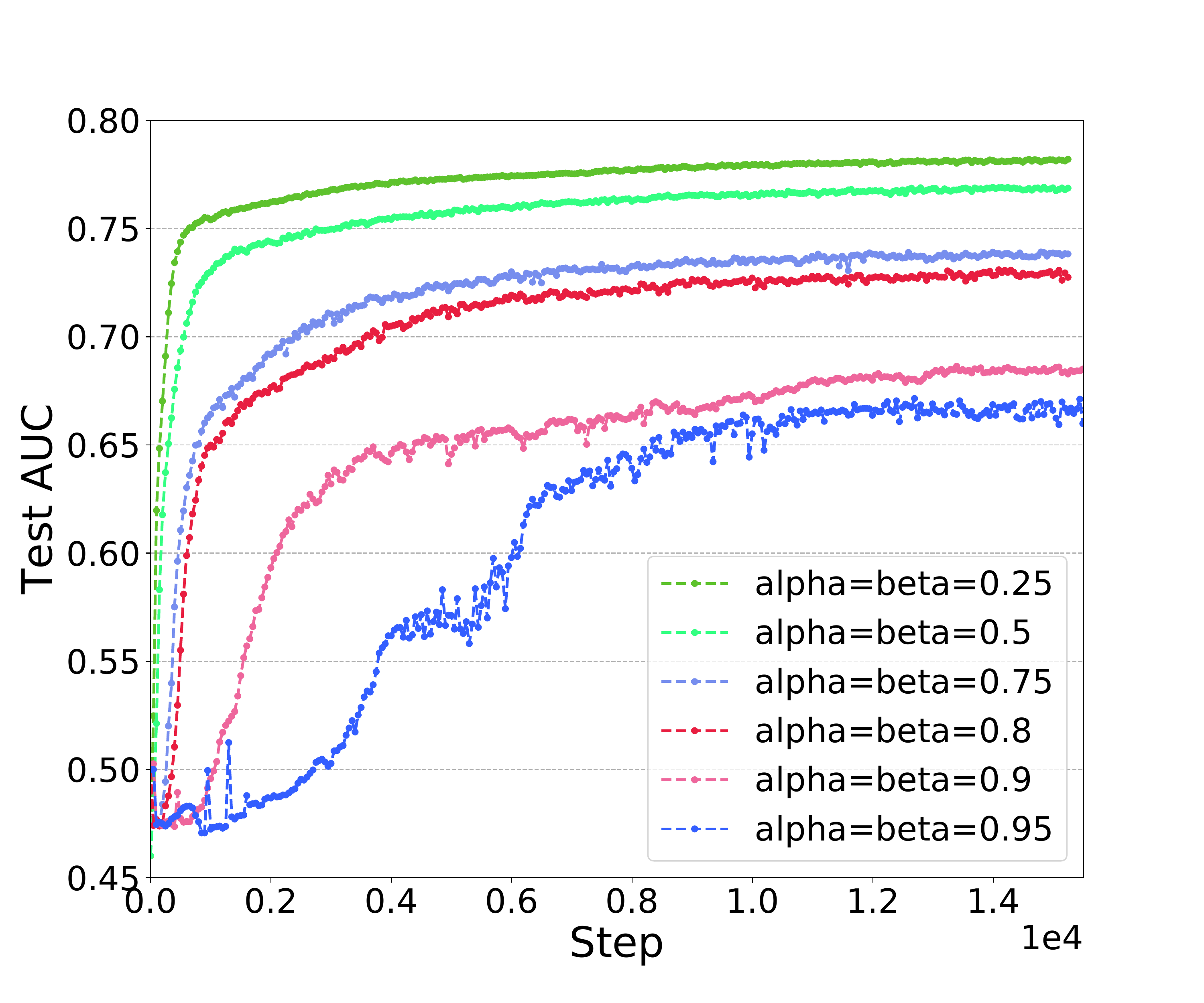}
      \caption*{ (j): Test AUC of \ourapp{}}
  \end{minipage} 
  
 \caption{Performance of different synthetic strategies on Criteo dataset. Performance is evaluated by test AUC and ACE. Figure (a) and (d): label synthetic strategies \synzerote{} and \synzerotr{} perform the best among all competitors. Figure (b) and (e): Even though \feaGaussian{} and \feaRandom{} can have comparable performance with \baselinepartial{}, both of them have no privacy guarantee. Figure (c) and (f): Each value in the x-axis represents a synthetic ratio for each party. $\alpha$ is the synthetic feature ratio for the passive party and $\beta$ is the synthetic label ratio for the active party. And $\alpha = \beta$.  Figure (h), (i), (j): convergence analysis of \ourapp{} with different $\alpha$ and $\beta$.
 }
 \label{fig:criteo_utility_privacy} 
\vspace{-0.15in}
\end{figure*}

\subsection{Simulations with Synthetic Features Only}

We conduct experiments to simulate the scenario that only the passive party has to generate synthetic features for unknown instances during training. The active party owns all the sample IDs and it uses real labels for all training instances. We provide other two strategies for the passive party to generate synthetic cut layer embeddings instead of raw features as comparison partners. 1) \texttt{\feaGaussian}: Inside a mini-batch, each synthetic cut layer embedding $\widetilde{f(x)}$ is generated from  {$\calN(0, {(s/d)\cdot \max_{i \in I_p}^B \norm{f(x_i)}_2^2}I_{d \times d})$}, where $B$ is the mini-bath size and $s$ is tuned as $1$. 2) \texttt{\feaRandom}: Each $\widetilde{f(x)}$ is generated from {$\calN(\mu, \sigma^2)$}, where $\mu$ and $\sigma$ are the current moving average of real cut layer embeddings' mean and standard deviation respectively. Following the similar setting as described in in Section \ref{sec:exp_synthetic_label_only}, we also provide two baselines: \texttt{\baselinepartial{}} and \texttt{\baselinefull{}} and two variant of \feaSampling{} with logits calibration: \texttt{\feaSamplingTe{}} and \texttt{\feaSamplingTr{}}. Details of the logits calibration can be found in Appendix \ref{sec:logits_calibration_synthetic_features}. 

We vary the synthetic feature ratio $\alpha$ to in our experiments. As shown in Figure \ref{fig:criteo_utility_privacy} (b) and (e), \feaGaussian{} and \feaRandom{} have comparable AUC with \baselinepartial{}. However, 
both of them have no privacy guarantee, since their attack AUC are close to $1.0$ as shown in Table \ref{tab:syn_fea_criteo_leak_auc}. In addition, \feaSamplingTe{} and \feaSamplingTr{} reduce the ACE significantly and have comparable ACE with \baselinepartial{}. Regarding AUC, \feaSamplingTr{} drops about $4.9\%$ and $0.77\%$ in comparing with \baselinepartial{} when $\alpha = 0.95$ and $\alpha = 0.5$ respectively. 
Meanwhile, the attack AUC of \feaSampling{} and its variants all fall in $[0.5, 0.507]$ (almost $0.5$), as shown in Table \ref{tab:syn_fea_criteo_leak_auc}. 



\subsection{Simulations with Both Synthetic Labels and Features}

We conduct experiments with a more realistic scenario: both parties provide synthetic data. We vary $\beta$ in the passive party and  $\alpha$ in the active party and compare our proposed approach \ourappvanilla{} and \ourapp{} with baselines (\baselinefull{} and \baselinepartial{}). Both \ourappvanilla{} and \ourapp{} leverage \synzero{} to generate synthetic labels and \feaSampling{} to generate synthetic features. The difference is that \ourapp{} leverages logits calibration to reduce ACE caused by synthetic features and labels, while \ourappvanilla{} only deals with calibration with \synzerote{}. Both parties may add dummy sample IDs to perturb the union of training samples. Hence, training instances with synthetic label and feature simultaneously are allowed in our experiments.

Due to space limit, we only report our experimental results with $\alpha = \beta$. We find the same conclusion under other settings. The real data ratio used to train the model is $(1-\alpha)(1-\beta)$. For example, when the synthetic ratio per party is $0.95$, \baselinepartial{} leverages $0.25\%$ data to train. As shown in Figure \ref{fig:criteo_utility_privacy} (c) and (f), \ourapp{}  and \ourappvanilla{} achieve similar AUC. However, \ourapp{} can reduce the ACE significantly in comparing with \baselinepartial{}. As shown in Table ~\ref{tab:our_approach_auc_drops}, \ourapp{} has a reasonable AUC drop in comparing with \baselinepartial{}. 

\begin{table}[t]\centering
\caption{Attack AUC of different synthetic feature generation strategies on Criteo.}\label{tab:syn_fea_criteo_leak_auc}
\scriptsize
\begin{tabular}{c|ccccc}\toprule
$\alpha$ = 0.5  &fea-Gaussian &fea-random &fea-sampling &fea-sampling-tr &fea-sampling-te \\\midrule
Attack AUC  &0.9766 &0.982 &0.5059 &0.5067 &0.5064 \\
\bottomrule
\end{tabular}
 \vspace{-0.15in}
\end{table}


Figure \ref{fig:criteo_utility_privacy} (h), (i), and (j) shows the convergence analysis for \ourapp{} with different $\alpha$ and $\beta$. Both training loss\footnote{The training loss is computed as the cross entropy with synthetic labels as the ground truth.} and test loss decease to stable points for all settings. 
Only when $\alpha = \beta = 0.95$, we see the test loss is fluctuating at some points. Another observation is that larger $\alpha$ and $\beta$ has slower convergence as we see that $\alpha = \beta = 0.95$ is the slowest one to achieve a stable test AUC.

\begin{table}[h]\centering
\caption{AUC change compared to baseline-partial on Criteo.}\label{tab:our_approach_auc_drops}
\scriptsize
\begin{tabular}{c|cccccccccc}\toprule
\textbf{ $\alpha = \beta$} &{0.05} &{0.1} &{0.25} &{0.5} &{0.6} &{0.75} &{0.8} &{0.9} &{0.95} \\\midrule
$\Delta$ AUC &-0.076\% &-0.165\% &-0.408\% &-1.259\% &-1.707\% &-3.081\% &-3.566\% &-4.924\% &-5.119\% \\
\bottomrule
\end{tabular}
\vspace{-0.1in}
\end{table}
\section{Related Work}

\paragraph{Vertical Federated Learning.} 
Vertical Federated Learning~\citep{vgsr18,gr18,akk+20,csm+20} considers the scenario where multiple parties own different attributes of the same entities.
Existing vFL protocols~\citep{yrzl19,cjsy20,wcxco20} are all based on identifying and then training the model on the intersection. Therefore it would unavoidably leak the intersection membership. Our work is also closely related to Split Neural Network (SplitNN)~\citep{vgsr18,gr18,akk+20,csm+20}, which is another framework that can be used in vFL scenario. 


\paragraph{Private Set Union.}
Despite extensive study on PSI~\citep{dt10,dkt10,hek12,hv17,prty19}, little is known on PSU until the work of Kissner and Song~\citep{ks05}.
Existing work on PSU employs advanced techniques in multi-party computation including additive homomorphic encryption~\citep{ks05,f07,dc17}, Bloom filter~\citep{dc17}, garbling circuits~\citep{ba16}, and oblivious transfer~\citep{ba16,krnw19}. Different from prior work, we compute the union of hash function of the ID sets so that the raw ID information is not used in the training procedure.
Our PSU protocol is based on the Diffie–Hellman key exchange scheme~\citep{dh76}, and is inspired by the work of Buddhavarapu et.al. \citep{bkm++}.
Compared to \citep{bkm++}, our protocol is conceptually cleaner because we decouple the union computation from identifying the hash mapping.



\section{Conclusion}
We propose a novel vFL framework that addresses the intersection membership leakage problem which currently prevents many privacy-sensitive organizations from adopting vFL. Our framework uses PSU to generate a union of all samples, which does not reveal intersection membership information to all parties. In addition, 
synthetic data generation strategies are proposed to handle samples belong to the union but not the intersection. We experimentally show that our method can protect the intersection membership without a  significant drop in model performance.


\paragraph{Limitations.} We point out two limitations in our framework. \textit{First}, compared to the  PSI approach, our PSU increases the training sample size from the intersection size ($|\mathcal{I}_p \cap \mathcal{I}_a|$) to the union size ($|\mathcal{I}_p \cup \mathcal{I}_a|$). Therefore, when the union is significantly larger than the intersection, our method would incur a non-negligible training cost compared to PSI. In practice, the larger party can choose to downsample its data to reduce training overhead. In addition, 
before the training starts, parties can estimate the additional training cost by measuring the size difference between intersection and union, and then choose between our framework and traditional PSI. \textit{Second},  synthetic label generation \synzero{} is specifically designed for unbalanced binary classification tasks such as online advertising and healthcare applications. We leave a more general strategy (other than \synpredprob{}) for balanced binary classification and multi-classification problems  as future work.




\bibliographystyle{abbrvnat}
\bibliography{references}


\clearpage
\appendix
\let\oldnl\nl
\newcommand{\nonl}{\renewcommand{\nl}{\let\nl\oldnl}}

\textbf{Appendix Outline:}

Section ~\ref{appendix:psu}: Security Analysis of Private Set Union

Section ~\ref{appendix:kl_proof}: Proof of Theorem \ref{theorem_min_kl}

Section ~\ref{sec:spectral_attack}: Spectral Attack

Section ~\ref{sec:training_serving}: Federated Model Training and Serving

Section ~\ref{sec:logits_shift_synthetic_label}: Logits Calibration for Synthetic Labels Only 

Section ~\ref{sec:logits_calibration_synthetic_features}: Logits Calibration for Synthetic Features Only

Section ~\ref{sec:ls_both}: Logits Calibration for Synthetic Both Features and Labels

Section ~\ref{sec:syn_label_pred_prob}: Synthetic Label Generation: \synpredprob

Section ~\ref{sec:exp_avazu}: Experimental Results on the Avazu Dataset

Section ~\ref{sec:data_setup_experimental_details}: Data Setup and Experimental Details

\clearpage

\section{Security Analysis of Private Set Union}\label{appendix:psu}
We first introduce some notations in multi-party computation adapted from~\cite{l17}.
The \textit{view} of a party $P$ (either $\actp$ or $\pasp$),
consists of the input of $P$ (either $\actid$ or $\pasid$),
the internal randomness used by $P$,
and all the messages $P$ receives during the MPC procedure.
Let $n=\max\{|\actid|,|\pasid|,\lceil \log p \rceil\}$ be the \textit{security parameter}, where $p$ is the prime number used in Algorithm \ref{alg:psu}.
A \textit{simulator} $S$ for $P$ is a probabilistic algorithm that takes the input of $P$  and $|\actid|, |\pasid|, |\actid\cup\pasid|$ as input,
runs in time polynomial in $n$,
and outputs a transcript that is indistinguishable to the view of $P$ for any polynomial time algorithm.

The idea behind simulation based proof is that for honest-but-curious malicious party,
all the information it could inspect is from its view.
Therefore,
if the view of malicious party is indistinguishable to the output of the simulator,
then any polynomial time malicious party cannot infer more information other than the input of the simulator.

Now we can state our formal result on the security of our PSU protocol.
We first present a simulator $S_{a}$ against $\actp$, which is described in Algorithm~\ref{alg:simulator_active}. The security guarantee is shown in Theorem \ref{thm:simulator_active}
\begin{algorithm}[!t]
\SetAlgoLined
\LinesNumbered
 \caption{Simulator $S_a$ against the active party $\actp$}\label{alg:simulator_active}
\KwIn{Prime $p=2q+1$, $\actid$, $|\pasid|$, $|\actid \cap \pasid|$.} 
\KwOut{View of $\actp$.}
\nonl \textbf{Simulate Initialization:}\\
\quad $S_a$ generates random $s_1,s_2,s_3\in \mathbb{Z}_q$.\\
\nonl \textbf{Simulate First Round Hashing:}\\
\quad $S_a$ faithfully computes $\actid^{s_1}$ as in Algorithm \ref{alg:psu}.\\
\quad $S_a$ randomly chooses $|\actid|$ elements $G\equiv\{g_1,g_2,\cdots,g_{|\actid|}\}$ from $QR(\mathbb{Z}_{p}^*)$.\\
\label{eq:first_round_hashing}\quad $S_a$ computes $\overline{\actid^{s_1t_1} }\equiv \{x^{s_1}|x\in G\}$, and randomly shuffles $\overline{\actid^{s_1t_1}}$.\label{eq:first_round_hashing}\\
\quad $S_a$ randomly choose $|\pasid|-|\actid \cap \pasid|$ elements $H$ from $QR(\mathbb{Z}_p^*)$.\\
\quad $S_a$ computes $\overline{\pasid^{t_1}}\equiv \{g_i\}_{i \in [|\actid\cap \pasid|]}\cup H$, and randomly shuffles $\overline{\pasid^{t_1}}$.\\
\quad $S_a$ faithfully computes $\overline{\pasid^{s_1t_1}}\equiv \{x^{s_1}|x\in \overline{\pasid^{t_1}}\}$.\\
\nonl \textbf{Simulate Second Round Hashing:}\\
\quad $S_a$ computes $\overline{(\actid\cup \pasid)^{s_1t_1}}$ by merging $\overline{\actid^{s_1t_1}}$ and $\overline{\pasid^{s_1t_1}}$.\\
\quad $S_a$ computes $\overline{(\actid\cup \pasid)^{s_1s_2s_3t_1}} \equiv \{x^{s_2s_3}|x\in \overline{(\actid\cup \pasid)^{s_1t_1}}\}$, and randomly shuffles it.\\
\quad $S_a$ randomly generates $m = |\actid\cup \pasid|$ elements $K\equiv \{k_1,\cdots,k_m\}$ from $QR(\mathbb{Z}_p^*)$.\\
\quad $S_a$ computes $\overline{(\actid\cup \pasid)^{s_1s_2s_3t_1t_2t_3}} \equiv \{x^{s_1s_2s_3}|x\in K\}$ and sorts it.\\
\textbf{Simulate Computing Private Hashing:}\\
\quad \For{$i=1,2,\cdots,|\actid|$}{
$S_a$ chooses $x$ as the $i$-th element in $\actid$.\\
$S_a$ faithfully computes $y=x^{s_2}$.\\
$S_a$ computes $\overline{z}=k_i^{s_2}$.
$S_a$ computes $\overline{x^{s_1s_2s_3t_1t_2t_3}}=(\overline{z})^{s_1s_3}$.
}
\quad \For{$i=1,2,\cdots,|\pasid|$}{
$S_a$ randomly chooses $\overline{x^{t_2}}$ from $QR(\mathbb{Z}_p^*)$.
}
\end{algorithm}
\begin{theorem}\label{thm:simulator_active}
Assume that the Decisional Diffie–Hellman problem is hard for $QR(\mathbb{Z}_p^*)$.
Then the distribution of the view of $\actp$ by running Algorithm \ref{alg:psu} is indistinguishable to the output of $S_a$ for any distinguisher that runs in time polynomial in $n$.
\end{theorem}
\begin{proof}
The proof follows from~\cite{bkm++} and adapts a hybrid argument.
Notice that the output of $S_a$ is essentially replacing the messages from $\pasp$ with proper random elements from $\mathbb{Z}_q$.
Therefore we can build a sequence of hybrid views,
where we start with the view of $\actp$ by really executing Algorithm \ref{alg:psu},
and we end with the view that is the output of $S_a$.
Furthermore,
neighboring hybrid views only differ at one place.
We are going to argue that distinguishing neighboring views can be reduced to the DDH problem,
hence proves the security.

Let us see a concrete example. Assume that for two neighboring views $v_i$, $v_{i+1}$,
the only difference is that in Line \eqref{eq:first_round_hashing},
for one single $x\in \actid-\pasid$,
$v_i$ uses $x^{s_1t_1}$,
while $v_{i+1}$ uses $g^{s_1}$ for some $g\in G$.
We now construct a reduction from DDH to distinguish these two views.
Recall that in DDH,
we are given a tuple $(h^a,h^b,h^c)$ where $h$ is a generator of $QR(\mathbb{Z}_p^*)$,
and we need to decide if $c=ab$, or $c$ is a random number from $\mathbb{Z}_q$.
Then we construct an instance of PSU as follows:
we set $x^{s_1}=h^b$ and $t_1=a$.
Then in $v_i$,
the tuple $(h^{t_1},x^{s_1},x^{s_1t_1})$ is distributed as $(h^{a},h^b,h^{ab})$,
while in $v_2$, $(h^{t_1},x^{s_1},\overline{x^{s_1t_1}})$ is distributed as $(h^{a},h^b,g^{s_1})$.
Since $g$ is uniform random, $g^{s_1}$ is distributed as $h^{c}$ for random $c\in \mathbb{Z}_q$.
Hence  if some distinguisher can distinguish $v_i$ and $v_{i+1}$,
then it can be used to solve DDH.
\end{proof}
We next present a simulator $S_{p}$ against $\pasp$ in Algorithm \ref{alg:simulator_passive}.
Since Algorithm \ref{alg:psu} is highly symmetric,
the corresponding security proof is quite similar and is hence omitted.
The security guarantee is summarized in the following theorem:
\begin{theorem}
Assume that the Decisional Diffie–Hellman problem is hard for $QR(\mathbb{Z}_p^*)$.
Then the distribution of the view of $\pasp$ by running Algorithm \ref{alg:psu} is indistinguishable to the output of $S_p$ for any distinguisher that runs in time polynomial in $n$.
\end{theorem}

\begin{algorithm}[!t]
\SetAlgoLined
\LinesNumbered
 \caption{Simulator $S_p$ against the passive party $\pasp$}\label{alg:simulator_passive}
\KwIn{Prime $p=2q+1$, $\pasid$, $|\actid|$,  $|\actid \cap \pasid|$.} 
\KwOut{View of $\pasp$.}
\nonl \textbf{Simulate Initialization:}\\
\quad $S_p$ generates random $t_1,t_2,t_3\in \mathbb{Z}_q$.\\
\nonl \textbf{Simulate First Round Hashing:}\\
\quad $S_p$ randomly chooses $|\actid|$ elements $G\equiv\{g_1,g_2,\cdots,g_{|\actid|}\}$ from $QR(\mathbb{Z}_{p}^*)$.\\
\quad $S_p$ computes $\overline{\actid^{s_1} }\equiv \{x|x\in G\}$, and randomly shuffles $\overline{\actid^{s_1}}$.\\
\quad $S_p$ faithfully computes $\overline{\actid^{s_1t_1}} \equiv \{x^{t_1}|x\in \overline{\actid^{s_1}}\}$ and randomly shuffles it.\\
\quad $S_p$ randomly choose $|\pasid|-|\actid \cap \pasid|$ elements $H$ from $QR(\mathbb{Z}_p^*)$.\\
\quad $S_p$ faithfully computes $\pasid^{t_1}$ as in Algorithm \ref{alg:psu}.\\
\quad $S_p$ computes $\overline{\pasid^{s_1t_1}}\equiv \{g_i^{t_1}\}_{i\in [|\actid\cap\pasid|]}\cup H$ and randomly shuffles it.\\
\nonl \textbf{Simulate Second Round Hashing:}\\
\quad $S_p$ randomly generates $m = |\actid\cup \pasid|$ elements $K\equiv \{k_1,\cdots,k_m\}$ from $QR(\mathbb{Z}_p^*)$.\\
\quad $S_p$ computes $\overline{(\actid\cup \pasid)^{s_1s_2s_3t_1}} \equiv \{x^{t_1}|x\in K\}$, and randomly shuffles it.\\
\quad $S_p$ computes $\overline{(\actid\cup \pasid)^{s_1s_2s_3t_1t_2t_3}} \equiv \{x^{t_2t_3}|x\in \overline{(\actid\cup \pasid)^{s_1s_2s_3t_1}}\}$ and sorts it.\\
\textbf{Simulate Computing Private Hashing:}\\
\quad \For{$i=1,2,\cdots,|\actid|$}{
$S_p$ randomly chooses $\overline{x^{s_2}}$ from $QR(\mathbb{Z}_p^*)$.
}
\quad \For{$i=1,2,\cdots,|\pasid|$}{
$S_p$ chooses $x$ as the $i$-th element in $\pasid$.\\
$S_p$ faithfully computes $y=x^{t_2}$.\\
$S_p$ computes $\overline{z}=k_i^{t_2}$.
$S_p$ computes $\overline{x^{s_1s_2s_3t_1t_2t_3}}=(\overline{z})^{t_1t_3}$.
}
\end{algorithm}

\section{Proof of Theorem \ref{theorem_min_kl}
}\label{appendix:kl_proof}
In this section we present the proof of Theorem \ref{theorem_min_kl}.
\begin{proof}
From the definition of KL-divergence, we have
\begingroup
\allowdisplaybreaks
\begin{align*}
  \arg\min_{D'} \text{KL}(D||D')= & ~\arg\min_{p\in \Delta_X} \int_{x,y}D(x,y)\log \frac{D(x, y)}{p(x)q(y)} dxdy\\
  = & ~\arg\max_{p\in \Delta_X} \int_{x,y}D(x,y)\log p(x)q(y) dx dy\\
  = &~\arg\max_{p\in\Delta_X}\left(\int_{x,y} D(x,y)\log p(x) dx dy+\int_{x,y} D(x,y)\log q(y) dx dy\right)\\
  = &~ \arg\max_{p\in\Delta_X}\left(\int_{x} \left(\int_yD(x,y))\log p(x) dy\right)dx\right) \\  = & ~\arg\min_{p\in\Delta_X}-\int_{x} p^* (x)\log p(x) dx\\  = &~\arg\min_{p\in\Delta_X}\int_{x} p^* (x)\log \frac{p^* (x)}{p(x)} dx\\ 
  = & ~\arg\min_{p\in\Delta_X} \text{KL}(p^* (x)||p(x))  
\end{align*}
where the fourth line follows from Fubini's theorem.
Then the claimed result follows from the well known fact that KL-divergence is non-negative and is 0 when $p=p^*$.
\end{proof}
\endgroup

\section{Spectral Attack}\label{sec:spectral_attack}
In this section we present the theoretical background of the spectral attack.
Spectral attack is a singular value decomposition (SVD) based outlier detection method introduced by Tran, Li and Madry~\cite{tlm18}.
In particular,
they show that
\begin{lemma}[Lemma 3.1, Definition 3.1 in \cite{tlm18}]\label{lem:spectral_attack}
Fix $0<\epsilon<\frac{1}{2}$.
Let $D$, $W$ be two distributions over $\mathbb{R}^d$ with mean $\mu_D,\mu_W$ and covariance matrices $\Sigma_D,\Sigma_W$.
Let $F$ be a mixture distribution given by $F=(1-\epsilon)D+\epsilon W$.
If $\|\mu_D-\mu_W\|_2^2\geq \frac{6\sigma^2}{\epsilon}$,
then the following statement holds:
let $\mu_F$ be the mean of $F$ and $v$ be the top singular vector of the covariance matrix of $F$,
then there exists $t$>0 so that 
\begin{align*}
    \Pr_{X\sim D}[|\langle X-\mu_F,v\rangle |>t] < & ~\epsilon,\\
    \Pr_{X\sim W}[|\langle X-\mu_F,v\rangle |<t]< & ~\epsilon.
\end{align*}
\end{lemma}

This gives us one attack method to distinguish the distribution $D$ of real data between the distribution $D'$ of synthetic data.
Lemma \ref{lem:spectral_attack} suggests that if the mean of $D$ and $D'$ are far away from each other,
then we can use $|\langle X-\mu_F,v\rangle |$ as the indicator to distinguish $D$ and $D'$.

In experiments,
for each mini-batch,
we compute the attack AUC for the passive party $\pasp$  when the synthetic label ratio $\beta>0$ and for  the active party $\actp$ when the synthetic feature ratio $\alpha>0$.
For $\pasp$,
we consider the distribution over the gradients $g$ sent back by $\actp$.
For $\actp$,
we consider the joint distribution over the intermediate embeddings $f$ sent from $\pasp$ and the labels $y$.
We estimate $\mu_F$ and $v$ by computing the empirical mean and covariance matrix with data in the mini-batch.
Then we are able to compute the score $|\langle X-\mu_F,v\rangle |$ for each sample in the batch.
Finally,
we compute the attack AUC as how these scores predict if each sample is synthetic or not.

\section{Federated Model Training and Serving}
\label{sec:training_serving}

After we leverage PSU to align the dataset and provide synthetic data using strategies as described in previous subsection, we then introduce how to do the federated model training and serving. In this paper we focus on two parties learning a model for a binary classification problem over the domain $\calX \times \cbrck{0, 1}$. Here the passive and active parties want to learn a composition model $h \circ f$ jointly, where the raw features $X$ and $f: \calX \to \Real^d$ are stored on the passive party side while the labels $y$ and $h: \Real^d \to \Real$ is on the active party side. Let $\ell = h(f(X))$ be the logit of the positive class where the positive class's predicted probability is given by the sigmoid function. The loss of the model is given by the cross entropy. 

To train the model using gradient descent, the passive party  computes $f(X)$ with both synthetic and real raw features $X$ and sends it to the active party who will then complete the rest of computation (Forward in Table \ref{tab:model_training}).\footnote{
For the simplifying of the notation and derivation, we add no additional features in the active party to compute the logit. 
The data generation strategies can be adapted for other complicated settings.
}  Then the active party starts the gradient computation process by first computing the gradient of the loss with respect to the logit $\frac{dL}{d\ell} = (\overline{p}_1 - y)$. Here $y$ contains both synthetic and real labels. Using the chain rule, the active party can then compute the gradient of $L$ with respect to $h$'s parameters through $\ell$. In order to allow the passive party to learn $f$, the active party also computes the gradient of $L$ with respect to the input of the function $h$. We denote this gradient by $g$  (equality by chain rule). After receiving $g$ sent from the active party, the passive party can compute the gradient of $L$ w.r.t. $f$'s parameters (Backward in Table \ref{tab:model_training}). 

  \begin{table*}[!ht]
  \centering
    \begin{tabular}{ cccccccccc }
      { Forward:} & $X$ & $\xrightarrow{\bm{f}}$ & $\bm{f}(X)$ & $\stackrel{\textrm{\scriptsize comm.}}{\Longrightarrow}$ & $\bm{f}(X)$ & $\xrightarrow{\bm{h}}$ &$\ell=h(f(X))$ & $\rightarrow$ & $\stackrel{\stackrel{\mbox{$y$}}{\downarrow}}{L}$ \\
      { Backward:} & & & $g$ & $\stackrel{\textrm{\scriptsize comm.}}{\Longleftarrow}$ & $g \coloneqq \nabla_{\bff(X)}L $ & $\leftarrow$ & $\frac{dL}{dl} = (\overline{p}_1 - y)$ & $\leftarrow$ & $1$ \\
      & & & $\downarrow$ & & & & $\downarrow$ & & \\
      & & & $f$'s param & & & & $h$'s param & & 
    \end{tabular}
    \caption{Communication diagram of model training ($\leftarrow$ and $\downarrow$ represent gradient computation using the chain rule).}
    \label{tab:model_training}
    \vspace{-0.15in}
  \end{table*}

When $B$ examples are forwarded as a batch, the communicated features $f(X)$ and gradients $g$ will both be matrices of shape $\Real^{B \times d}$ with each row belonging to a specific example in the batch. It is important to note that here the gradients as rows of the matrix are gradients of the loss with respect to different examples' intermediate computation results but not the model parameters; therefore, no averaging over or shuffling of the rows of the matrix can be done prior to communication for the sake of correct gradient computation of $f$'s parameters on the passive party side.

For model inference,
suppose that the passive party would like to know the prediction of instance $i$ from the active party. The passive party only feeds $i$'s raw feature $x$ and computes $f(x)$. 
If no additional features are added in the active party to compute the logit,
then the active party directly computes $p_i = 1/(1 + \exp(-h(f(x))))$ and sends $p_i$ to the passive party without knowing any ID information of $i$. 
Otherwise if additional features are needed for active party, the passive party can hide instance $i$ in a batch which both parties have agreements on. Then the serving process will be the as same as the forward process during the training phase. After receiving a batch of predictions from the active party, the passive party can select the corresponding prediction for $i$.

\section{Logits Calibration for Synthetic Labels Only}
\label{sec:logits_shift_synthetic_label}


Here we present the logits shift method inspired by~\cite{Chapelle2015ads}.
In this section we consider the scenario
where only the active party needs to provide synthetic labels,
and the synthetic labels are generated by \synzero{}.
Let $D$ be the ground truth distribution over $\mathcal{X}\times \mathcal{Y}$.
Let $D'$ be the distribution of data after introducing synthetic data.
Let $\delta^a_{id} = \mathbf{1}_{id\in \actid}$,
and $p_a$ be the probability of $id\in \actid$.
We make the assumption that the underlying ground-truth label distribution is independent on whether $id\in \actid$.
Then we have
\begin{equation}\label{eq:logits_shift_label}
   \begin{aligned}
    &~D'(y=1|X=x)\\
    = & ~\sum_{\delta^a_{id}=0}^{1}\Pr[\delta^a_{id} ] D'(y=1|X=x,\delta^a_{id})\\
    = & ~\Pr[\delta^a_{id}=1] D'(y=1|X=x,\delta^a_{id}=1 )+
    \Pr[\delta^a_{id}=0 ] D'(y=1|X=x,\delta^a_{id}=0)\\
    = & ~p_a D(y=1|X=x).
\end{aligned} 
\end{equation}








There are two ways to interpret Eq. \eqref{eq:logits_shift_label}:
\begin{itemize}
    \item \synzerote{}: We think of $D'$ as the model output,
    and do not modify the training phase.
    We compute $D$ according to Eq.~\eqref{eq:logits_shift_label} to evaluate on the test dataset.
    \item \synzerotr{}: We think of $D$ as the model output,
    and use it to evaluate in the test phase.
    In the training phase,
    we compute $D'$ according to Eq.~\eqref{eq:logits_shift_label} to compute the loss and corresponding model updates.
\end{itemize}

Empirically we found both logits shift can reduce the ACE significantly in comparing with the vanilla \synzero{} without any logits shift. 

\section{Logits Calibration for Synthetic Features Only}
\label{sec:logits_calibration_synthetic_features}







In this section we consider the scenario
where only the passive party needs to provide synthetic features,
and the synthetic features are generated by \feaSampling.
Let $D$ be the ground truth distribution over $\mathcal{X}\times \mathcal{Y}$.
Let $D'$ be the distribution of data after introducing synthetic data.
Let $\delta^p_{id} = \mathbf{1}_{id\in \pasid}$,
and $p_p$ be the probability of $id\in \pasid$.
We make the assumption that the underlying ground-truth label distribution is independent on whether $id\in \pasid$.
Then we have
\begin{equation}\label{eq:logits_shift_features}
   \begin{aligned}
    &~D'(y=1|X=x)\\
    = & ~\sum_{\delta^p_{id}=0}^{1}\Pr[\delta^p_{id} ] D'(y=1|X=x,\delta^p_{id})\\
    = & ~\Pr[\delta^p_{id}=1] D'(y=1|X=x,\delta^p_{id}=1 )+
    \Pr[\delta^p_{id}=0 ] D'(y=1|X=x,\delta^p_{id}=0)\\
    \\
    = & ~p_p D(y=1 | X =x) + (1-p_p) D(y=1).
\end{aligned} 
\end{equation}

In the experiment, $p_p$ is known by both parties, and $D(y=1)$ is estimated by the marginal distribution of true labels owned by $\actp$.
Similarly,
there are two ways to interpret Eq. \eqref{eq:logits_shift_features}:
\begin{itemize}
    \item \feaSamplingTe: We think of $D'$ as the model output,
    and do not modify the training phase.
    We compute $D$ according to Eq.~\eqref{eq:logits_shift_features} to evaluate on the test dataset.
    \item \feaSamplingTr: We think of $D$ as the model output,
    and use it to evaluate in the test phase.
    In the training phase,
    we compute $D'$ according to Eq.~\eqref{eq:logits_shift_features} to compute the loss and corresponding model updates.
\end{itemize}


\section{Logits Calibration for Synthetic Both Features and Labels} \label{sec:ls_both}
Here we present a toy example of our logits shift method.
Consider $\actid$ and $\pasid$ are constructed as follows: we start from a superset $\mathcal{I}$, and for each $id \in \mathcal{I}$, with probability $p_a$ we place $id$ into $\actid$, and with probability $p_p$ we place $id$ into $\pasid$ independently. 
Finally we run over vFL protocol by using $\mathcal{I}$ for scheduling.
$\actp$ uses fill-major strategy to generate synthetic labels, and $\pasp$ uses raw-sampling strategy to generate synthetic features.

Let $D$ be the ground truth distribution over $\mathcal{X}\times \mathcal{Y}$.
Let $D'$ be the distribution of data after introducing synthetic data.
Let $\delta^a_{id} = \mathbf{1}_{id\in \actid}$ 
and $\delta^p_{id} = \mathbf{1}_{id\in \pasid}$.
Then we have
\begin{equation}\label{eq:logits_shift}
   \begin{aligned}
    D'(y=1|X=x)
    = & ~\sum_{\delta^a_{id}=0}^{1}\sum_{\delta^p_{id}=0}^{1}\Pr[\delta^a_{id} ,\delta^p_{id} ] D'(y=1|X=x,\delta^a_{id} ,\delta^p_{id} )\\
    = & ~\Pr[\delta^a_{id}=1,\delta^p_{id}=1 ] D'(y=1|X=x,\delta^a_{id}=1 ,\delta^p_{id}=1 )\\&+
    \Pr[\delta^a_{id}=1 ,\delta^p_{id}=0 ] D'(y=1|X=x,\delta^a_{id}=0 ,\delta^p_{id}=0 )\\
    = & ~p_a p_p D(y=1|X=x)+p_a(1-p_p)D(y=1).
\end{aligned} 
\end{equation}

There are two ways to apply Eq.~\eqref{eq:logits_shift}.
We can either think of our model computing $D'$,
then on test data evaluation, we need to compute $D$ from $D'$;
or we can think of our model computing $D$,
then in the training process we need to compute $D'$ from $D$ and use $D'$ to compute the loss and update the model parameters.

\section{Synthetic Label Generation: \synpredprob}
\label{sec:syn_label_pred_prob}

In this section
we show the theoretical analysis of \synpredprob{}.
It is inspired by the so-called ``log-derivative trick'',
and the goal is to mitigate the effect of synthetic labels on the model updates.
Notice that in binary classification problems,
$\forall x\in \mathcal{X}$,
for any prediction model $p$,
we always have $\sum_{y_i=0}^{1}p(y_i|x)=1$.
Let $W$ be the weight on the cut layer,
and $(x_1,y_1),(x_2,y_2),\cdots$ be the training samples in a mini-batch.
We can split the indices into two parts $I_1$ and $I_2$,
where $I_1$ are the indices that $\actp$ has the labels,
and $I_2$ are the rest indices.
Since $\actp$ only needs to generate synthetic labels for indices in $I_2$,
we have
\begin{align*}
    \sum_{i\in I_2} \frac{\partial \sum_{y_i=0}^{1} p(y_i|x_i)}{\partial W}
    = & ~\sum_{i\in I_2}\sum_{y_i=0}^{1} \frac{\partial  p(y_i|x_i)}{\partial W}\\
    = & ~\sum_{i\in I_2}\sum_{y_i=0}^{1} p(y_i|x_i)\cdot \frac{1}{p(y_i|x_i)}\frac{\partial  p(y_i|x_i)}{\partial W}\\
    = & ~\sum_{i\in I_2}\sum_{y_i=0}^{1} p(y_i|x_i)\frac{\partial  \log p(y_i|x_i)}{\partial W}.
\end{align*}
On the other hand since $\sum_{y_i=0}^{1} p(y_i|x_i)=1$,
we have $\sum_{i\in I_2} \frac{\partial \sum_{y_i=0}^{1} p(y_i|x_i)}{\partial W}=0$.
This means if we sample synthetic label $y_i$ according to the distribution $q_{x_i}:=p(\cdot|x_i)$,
then
\begin{align*}
    \sum_{y_i=0}^{1} q_{x_i}(y_i)\frac{\partial  \log p(y_i|x_i)}{\partial W}=0.
\end{align*}
Notice that $\sum_{y_i=0}^{1 }q_{x_i}(y_i)\log p(y_i|x_i)$ can be interpreted as the expectation of the cross-entropy loss if we generate synthetic labels according to $q_{x_i}$.
This explains the intuition of \synpredprob{}:
we simply generate the synthetic label $y_i$ according to the model prediction $p(y_i|x_i)$.
If $q$ and $W$ are independent,
then the expected gradient with respect to the cross-entropy loss is exactly $0$,
namely the synthetic labels do not affect model updates on average.





\synpredprob{} performs better than \synone{}, \synknn{}, and \synposratio{}. 
However,
the variance could be large when we only sample once to generate the label.
To reduce the variance, we can increase the sample times and average the gradients generated from each sampled label. As a result, we can have better performance at the cost of worse privacy,
because the gradients corresponding to synthetic labels tend to be close to 0. 
Such a trade-off can be observed in Table~\ref{tab:label_predict}.
This experiment is conducted on the Avazu dataset with 2 epochs of training.


\begin{table}[!htp]\centering
\caption{ AUC and Attack AUC of \synpredprob{} on Avazu with $\beta = 0.5$. }\label{tab:label_predict}
\scriptsize
\begin{tabular}{c|ccccccccccc}\toprule
 Sample Times &1 &2 &3 &4 &5 &6 &7&8 &9 &10 \\ \midrule
 AUC & 0.7465&0.7482 &0.7482 &0.7486 & 0.7490 &0.7491&0.7489&0.7492 &0.7492 & 0.7493\\ \midrule
Attack AUC &0.5001 &0.5227 &0.5812 &0.6249  &0.6512 &0.6753&0.7041&0.7175 &0.7299&0.7438\\
\bottomrule
\end{tabular}
\end{table}
One advantage of \synpredprob{} is that it can not only be applied to both balanced and unbalanced binary classification problems but also multi-classification problems.

\section{Experimental Results on the Avazu Dataset}\label{sec:exp_avazu}
\begin{figure*}[ht!]
\vspace{-0.1in}
  \begin{minipage}[b]{0.33\linewidth}
  \centering
    \includegraphics[width=\linewidth]{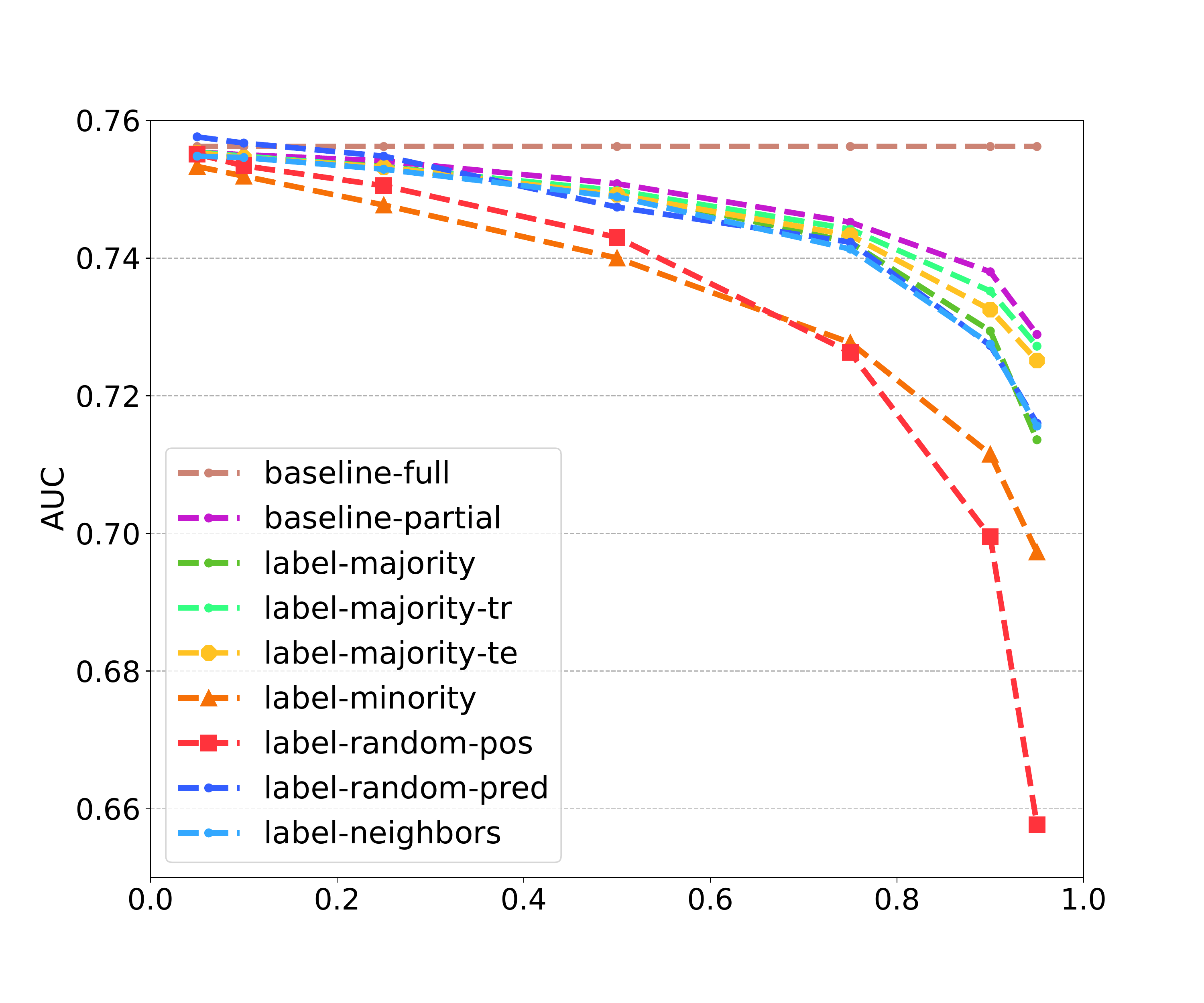}
      \caption*{ (a): Synthetic Label AUC}
  \end{minipage}
  \begin{minipage}[b]{0.33\linewidth}
  \centering
    \includegraphics[width=\linewidth]{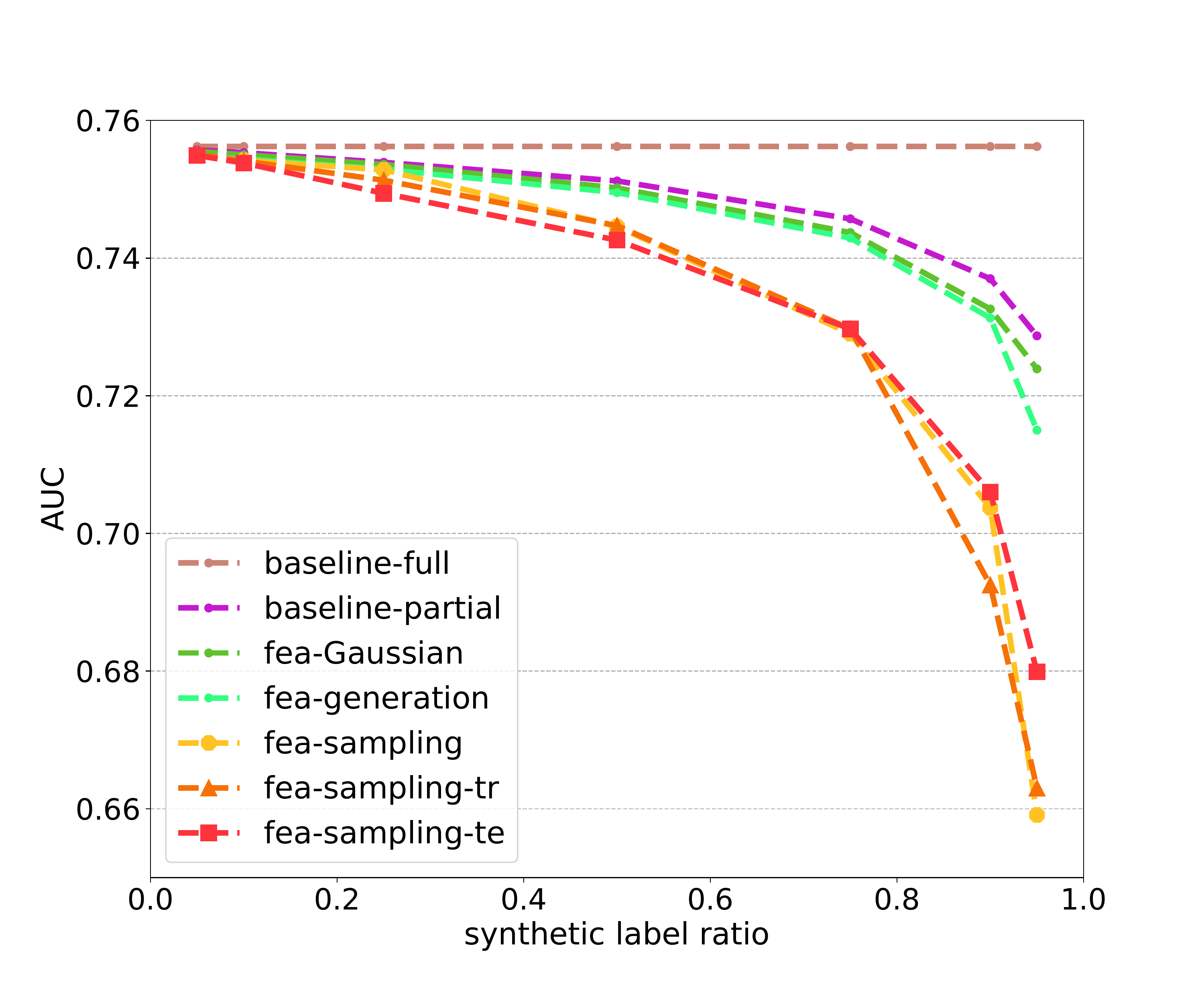}
    \caption*{(b): Synthetic Feature AUC}
  \end{minipage}
   \begin{minipage}[b]{0.33\linewidth}
  \centering
    \includegraphics[width=\linewidth]{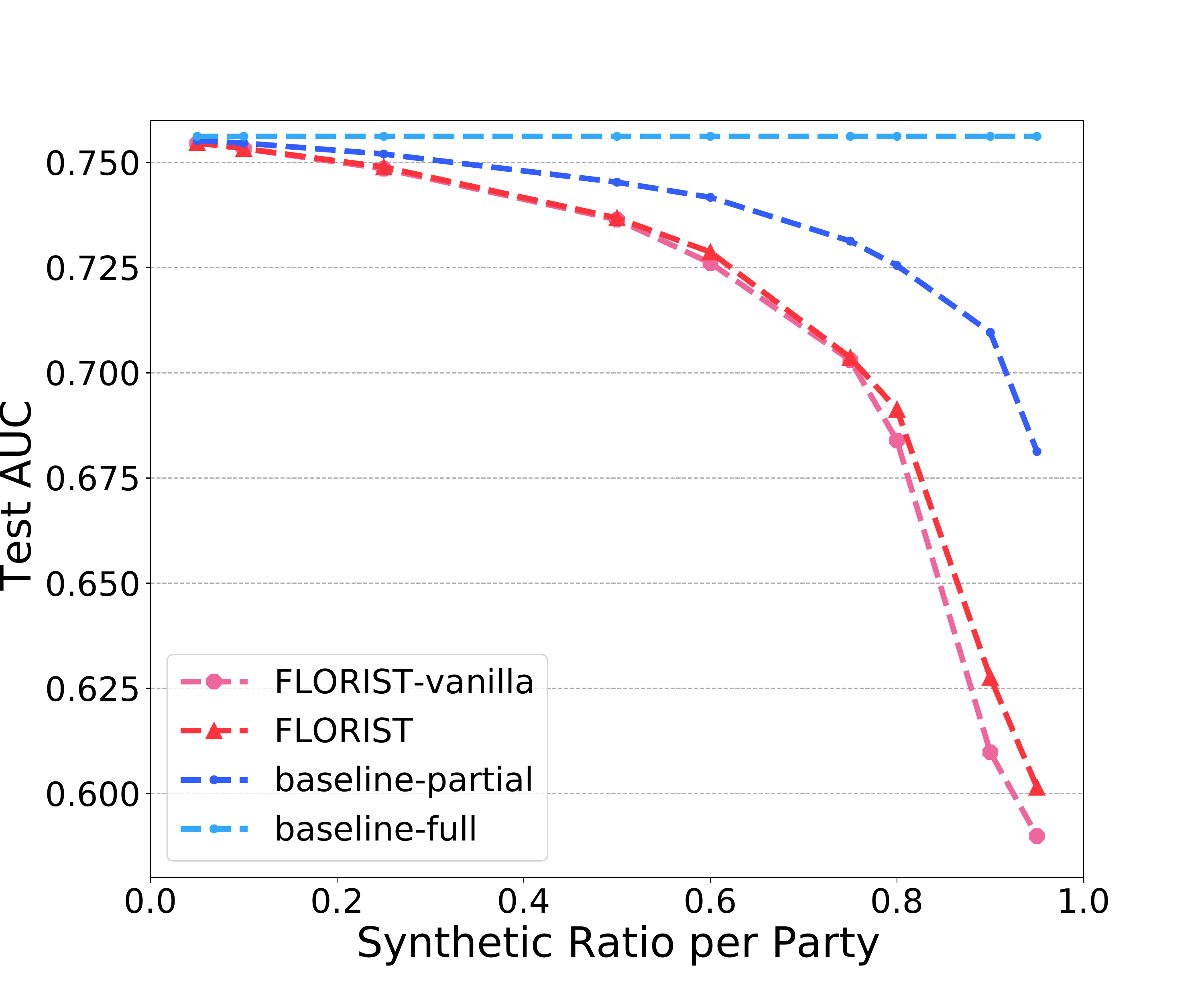}
      \caption*{(c): Synthetic Both AUC}
  \end{minipage}
  \medskip
    \begin{minipage}[b]{0.33\linewidth}
  \centering
    \includegraphics[width=\linewidth]{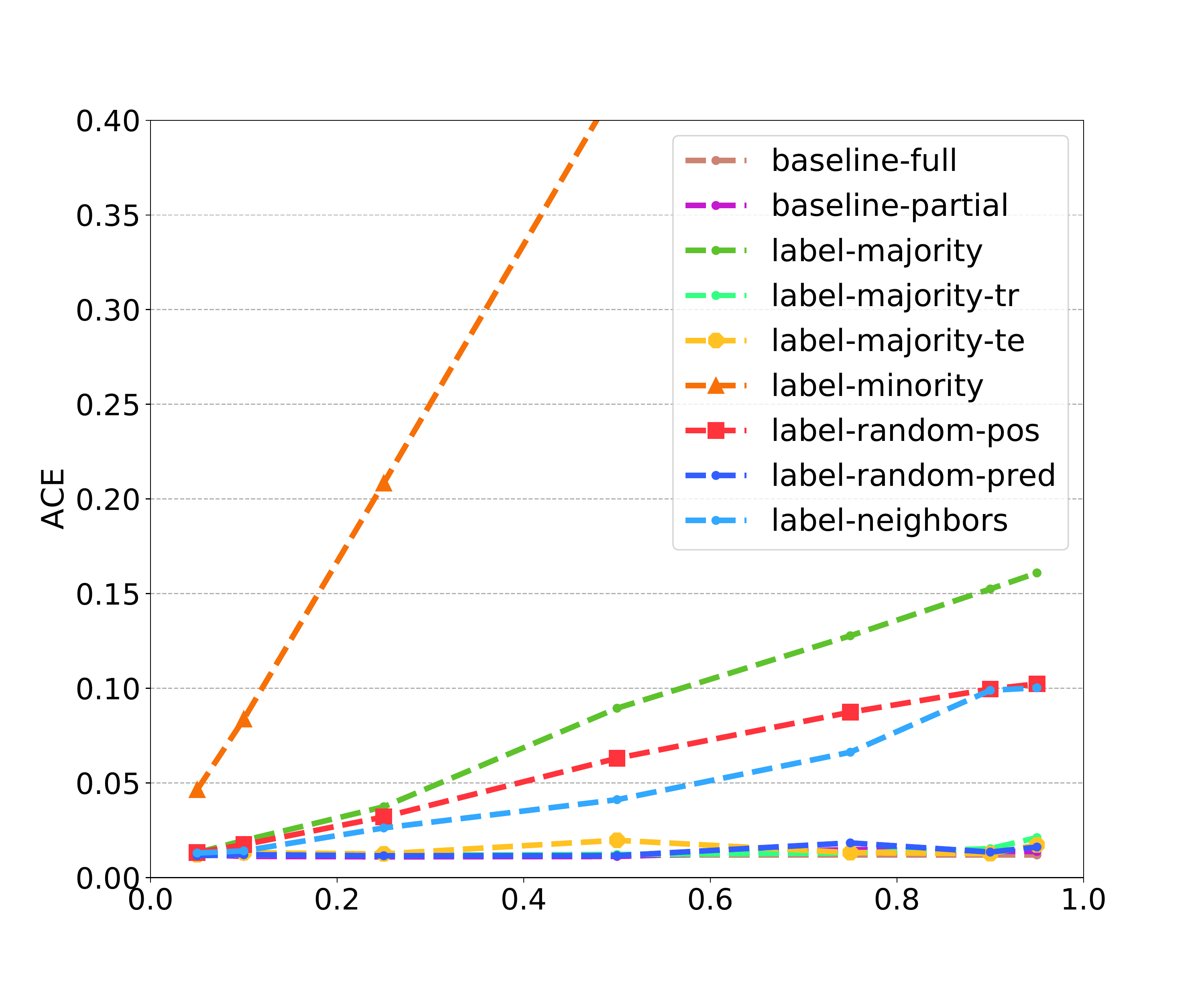}
      \caption*{(d): Synthetic Label ACE}
  \end{minipage}
  \begin{minipage}[b]{0.33\linewidth}
  \centering
    \includegraphics[width=\linewidth]{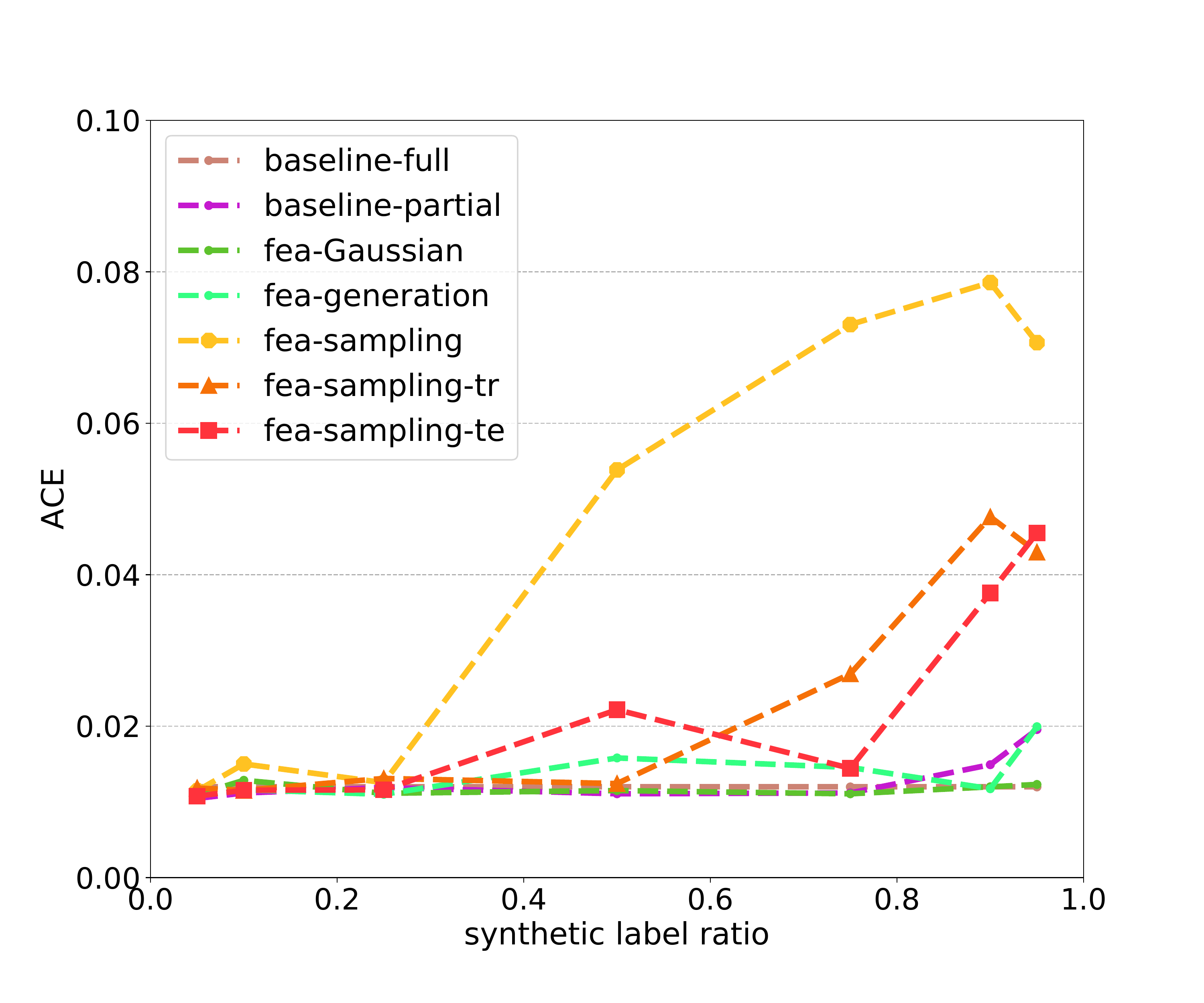}
    \caption*{(e): Synthetic Feature ACE}
  \end{minipage}
   \begin{minipage}[b]{0.33\linewidth}
  \centering
    \includegraphics[width=\linewidth]{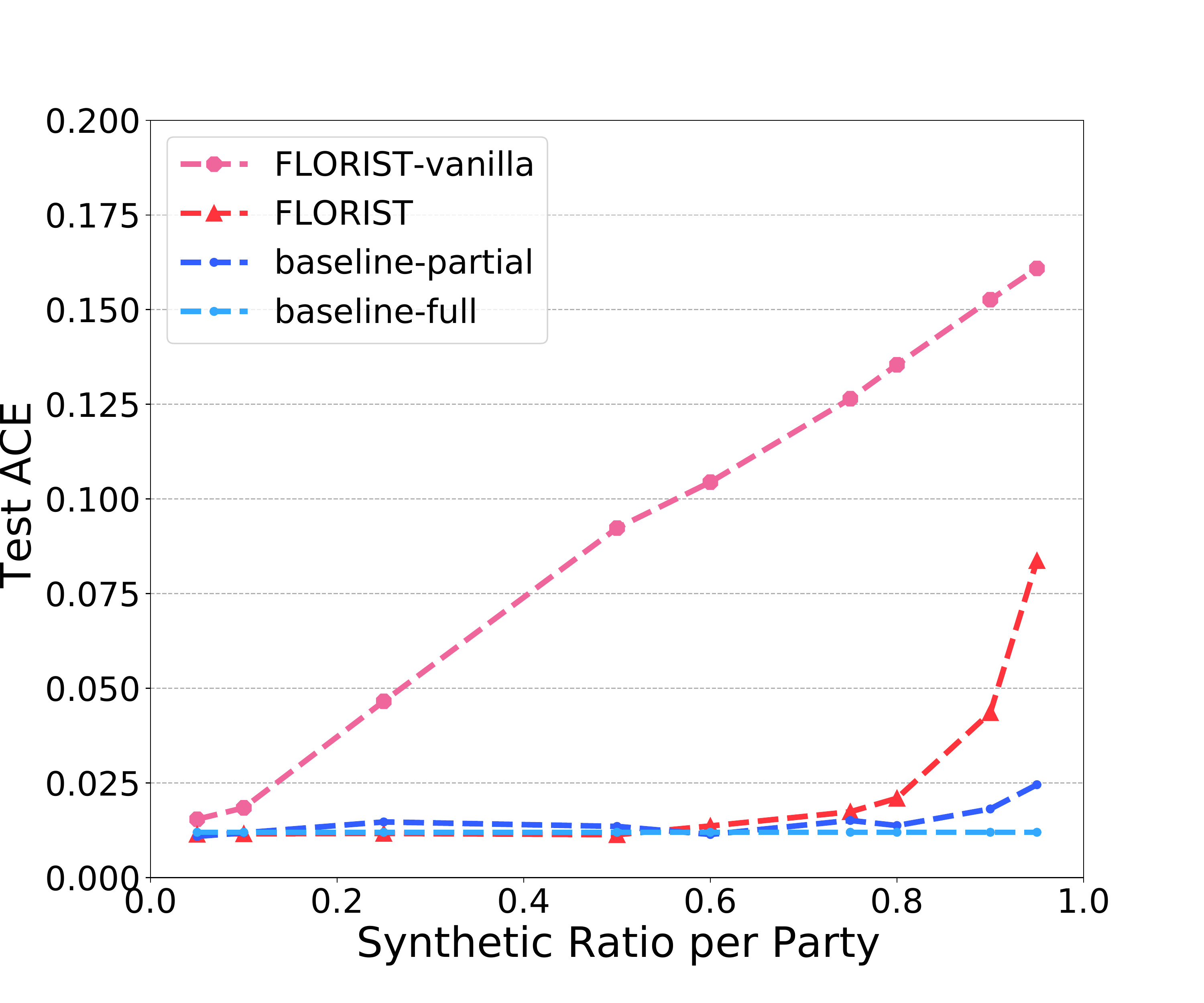}
      \caption*{(f): Synthetic Both ACE}
  \end{minipage}
   \caption{Performance of different synthetic strategies on  Avazu dataset. Performance is evaluated by test AUC and ACE. Figure (c) and (f): Each value in the x-axis represents a synthetic ratio for each party. $\alpha$ is the synthetic feature ratio for the passive party and $\beta$ is the synthetic label ratio for the active party. $\alpha = \beta$. 
 }
 \label{fig:avazu_utility_privacy} 
  \end{figure*}

In this section, 
we report the experiment result on Avazu dataset. 
Avazu is a dataset of click-through data with approximately 40 million entries (11 days of clicks/not clicks of Avazu data).
Overall,
the trends of all evaluation metrics are very similar to those in the experiments on Cretio reported in Section~\ref{sec:experiments}.

\subsection{Simulations with Synthetic Labels Only}
The experiment results on Avazu dataset of AUC and ACE are shown in Figure~\ref{fig:avazu_utility_privacy} (a) and (d) respectively.
The attack AUC can be found in Table \ref{tab:syn_label_avazu_leak_auc}.
We can see that \synzero{} can achieve similar AUC compared to \baselinepartial{},
while maintaining a relatively low attack AUC (about 0.55).
Furthermore,
by applying logits shift,
ACE can be lowered to the scale of \baselinepartial{}.
These experiment results reaffirm the validity of using \synzero{} with logits shift as the generation strategy for the active party.

\begin{table}[!ht]\centering
\caption{Attack AUC of different synthetic label generation strategies on Avazu.}\label{tab:syn_label_avazu_leak_auc}\scriptsize
\begin{tabular}{c|cccccccc}\toprule
$\beta=0.5$ &.-majority &.-majority-tr &.-majority-te &.-minority &.-random-pos &.-random-pred &.-neighbors \\\midrule
Attack AUC &
0.558341 &0.558277 &0.557430 &0.584588 &0.442587&0.500570&0.536879\\
\bottomrule
\end{tabular}
\vspace{-0.15in}
\end{table}

\subsection{Simulations with Synthetic Features Only}
The experiment results on Avazu dataset of AUC and ACE are shown in Figure~\ref{fig:avazu_utility_privacy} (b) and (e) respectively.
The attack AUC can be found in Table \ref{tab:syn_fea_avazu_leak_auc}.
Again we find that \feaGaussian{} and \feaRandom{} have comparable AUC with \baselinepartial{}.
However their attack AUC is very close to $1$,
meaning that they are not suitable for privacy-sensitive vFL.
On the other hand,
\feaSampling{} has a reasonable attack AUC,
and the drop on AUC is moderate compared to \baselinepartial{}.
For instance,
the AUC drops about $6.7\%$ and $0.88\%$ compared with \baselinepartial{} when $\alpha = 0.95$ and $\alpha = 0.5$ respectively.
We also observe that logits shift can significantly lower the ACE.
\begin{table}[!ht]\centering
\caption{Attack AUC of different synthetic feature generation strategies on Avazu.}\label{tab:syn_fea_avazu_leak_auc}
\scriptsize
\begin{tabular}{c|ccccc}\toprule
$\alpha$ = 0.5  &fea-Gaussian &fea-random &fea-sampling &fea-sampling-tr &fea-sampling-te \\\midrule
Attack AUC  &
0.896055 &0.897465 & 0.540238 &0.534171 &0.557020\\
\bottomrule
\end{tabular}
 \vspace{-0.15in}
\end{table}

\subsection{Simulations with Both Synthetic Labels and Features}
As in Section~\ref{sec:experiments},
 here we report our experimental results with $\alpha = \beta$.
 The experiment results of AUC and ACE are shown in Figure~\ref{fig:avazu_utility_privacy} (c) and (f) respectively.
 We find that by applying logits shift,
 \ourapp{} greatly reduces ACE compared to \ourappvanilla{}.
 We also report the AUC drop for \ourapp{} compared to \baselinepartial{}.
The drop is barely minimal in scenarios close to real application (like $\alpha=\beta=0.5$),
and is still affordable even when $\alpha,\beta$ is as large as $0.95$.
 \begin{table}[h]\centering
\caption{AUC change compared to baseline-partial on Avazu.}\label{tab:our_approach_auc_drops_avazu}
\scriptsize
\begin{tabular}{c|cccccccccc}\toprule
\textbf{ $\alpha = \beta$} &{0.05} &{0.1} &{0.25} &{0.5} &{0.6} &{0.75} &{0.8} &{0.9} &{0.95} \\\midrule
$\Delta$ AUC & -0.066\%
&-0.17\%
&-0.41\%
&-1.14\%
&-1.75\%
&-3.78\%
&-4.71\%
&-11.55\%
&-11.71\%\\
\bottomrule
\end{tabular}
\vspace{-0.1in}
\end{table}

\section{Data Setup and Experimental Details}
\label{sec:data_setup_experimental_details}

We first describe how we first preprocess each of the datasets. We then describe the model architecture used for each dataset. Finally, we describe what are the training hyperparameters used for each dataset/model combination and the total amount of compute required for the experiments.

\subsection*{Dataset preprocessing}

\paragraph{[Criteo]} Every record of Criteo has $27$ categorical input features and $14$ real-valued input features. We first replace all the \texttt{NA} values in categorical features with a single new category (which we represent using the empty string) and all the \texttt{NA} values in real-valued features by $0$. For each categorical feature, we convert each of its possible value uniquely to an integer between $0$ (inclusive) and the total number of unique categories (exclusive). For each real-valued feature, we linearly normalize it into $[0,1]$. We then randomly sample $10\%$ of the entire Criteo publicly provided training set as our entire dataset (for faster training to generate privacy-utility trade-off comparision) and further make the subsampled dataset into a 90\%-10\% train-test split.

\paragraph{[Avazu]} Unlike Criteo, each record in Avazu only has categorical input features. We similarly replace all \texttt{NA} value with a single new category (the empty string), and for each categorical feature, we convert each of its possible value uniquely to an integer between $0$ (inclusive) and the total number of unique categories (exclusive). We use all the records in provided in Avazu and randomly split it into 90\% for training and 10\% for test.

\subsection*{Model architecture details}

\paragraph{[Criteo]} We modified a popular deep learning model architecture WDL \cite{cheng2016wide} for online advertising. We first process the categorical features in a given record by applying an embedding lookup for every categorical feature's value. We use an embedding dimension of 4 for the deep part. After the lookup, the deep embeddings are then concatenated with the continuous features to form the raw input vectors for the deep part. The deep part processes the raw features using 4 ReLU-activated 128-unit MLP layers before producing a final logic value. The cut layer is after the output of the 2rd ReLU layer on the deep part.

\subsection*{ Model training details}

To ensure smooth optimization and sufficient training loss minimization, we use a slightly smaller learning rate than normal.

\paragraph{[Criteo]}
We use the Adam optimizer with with a batch size of 8,192 and a learning rate of $1e$$-4$ throughout the entire training of 3 epochs (approximately 15k stochastic gradient updates).

\paragraph{[Avazu]}
We use the Adam optimizer with a batch size of 8,192 and a learning rate of $1e$$-4$ throughout the entire training of 3 epochs (approximately 15k stochastic gradient updates).

We conduct our experiments over 8 Nvidia Tesla V100 GPU card. Each epoch of run of Avazu takes about 10 hours to finish on a single GPU card occupying 4GB of GPU RAM. Each epoch run of Criteo takes about 11 hours to finish on a single GPU card using 4 GB of GPU RAM. 
\end{document}